\documentclass{amsart}

\usepackage{amsmath,amssymb,amsfonts}
\usepackage{amsthm}
\usepackage{mathtools}

\usepackage{graphicx}    

\usepackage{subcaption}
\usepackage{bm}
\usepackage{array}
\usepackage{booktabs}
\usepackage{placeins}
\usepackage{epstopdf}
\usepackage{tabulary}
\usepackage{longtable}
\usepackage{etoolbox}

\BeforeBeginEnvironment{figure, table}{\vskip-2ex}
\AfterEndEnvironment{figure, table}{\vskip-10ex}
\DeclareGraphicsExtensions{.eps,.png,.pdf,.jpg}
\numberwithin{equation}{section}
\newtheorem{theorem}{Theorem}[section]
\theoremstyle{remark} \newtheorem{remark}[theorem]{Remark}
\newcommand{\eat}[1]{}

\begin{document}

\title{A Distance Function for Comparing Straight-Edge Geometric
  Figures}

\author{Apoorva Honnegowda Roopa} 

\author{Shrisha Rao}

\begin{abstract}
This paper defines a distance function that measures the dissimilarity
between planar geometric figures formed with straight lines. This
function can in turn be used in partial matching of different
geometric figures. For a given pair of geometric figures that are
graphically isomorphic, one function measures the angular
dissimilarity and another function measures the edge length
disproportionality. The distance function is then defined as the
convex sum of these two functions. The novelty of the presented
function is that it satisfies all properties of a distance function
and the computation of the same is done by projecting appropriate
features to a cartesian plane. To compute the deviation from the
angular similarity property, the Euclidean distance between the given
angular pairs and the corresponding points on the $y=x$ line is
measured. Further while computing the deviation from the edge length
proportionality property, the best fit line, for the set of edge
lengths, which passes through the origin is found, and the Euclidean
distance between the given edge length pairs and the corresponding
point on a $y=mx$ line is calculated. Iterative Proportional Fitting
Procedure (IPFP) is used to find this best fit line. We demonstrate
the behavior of the defined function for some sample pairs of figures.
\end{abstract}



\maketitle

\section{Introduction}
\label{sec:Introduction}

Two geometric figures can be said to be similar if one of the
geometric figures can be obtained by either squeezing or enlarging the
other.  This implies that the considered geometric figures need to
have equal number of vertices and edges, matching corresponding
angles, and a fixed proportionality between the corresponding edges.
This concept of similarity can be used for partial matching of
different geometric figures.

It is well known that geometric shapes and structures are important in
determining the behavior of chemical compounds.  This is true of
smaller molecules~\cite{Stoker2009} as well as larger macromolecules
such as DNA and RNA that are studied in
bioinformatics~\cite{Bourne2009}.  Molecular
geometry~\cite{Gillespie2005} is thus an important aspect of physical
and structural chemistry.  However, while it is also known that
similarity in structures often implies similar observed chemical
properties, there is yet no well defined mathematical approach for
comparing geometric shapes, and comparisons are made on an ad hoc
basis~\cite{Sen1987,Stashans2008}.  Such an approach as proposed here
would thus allow for a rigorous evaluation of such properties based on
the similarity of shapes with molecules with known properties.
Similarity in general has wide-ranging applications in many
domains~\cite{phdthesis}.

Image similarity and comparisons also play an important role in other
domains, such as in models of visual perception and object recognition
in humans as well as animals~\cite{Tarr1998,Blough2001}, finance and
economics~\cite{Vermorken2008}, computer vision~\cite{Sweeney2015},
and video analyses~\cite{Yu1996}.  In such contexts also there is much
scope for application of this work.

Existing theory in this matter is far from complete.  There are
heuristic approaches to morphological
similarity~\cite{Komosinski2001,Komosinski2011}, but no sound
mathematical basis for the detection of geometric similarity.
Geometric similarity is particularly important in engineering, in
comparing a model and its prototype~\cite{Heller2011,Pallett1961}, but
there however does not seem to be a proper universal measure of
geometric similarity.  The measure in common use in engineering is
merely scale-free identity, that all corresponding lengths should be
in the same ratio---there is thus no way to properly measure inexact
similarity, or to quantitatively state that a figure is more similar
to a reference figure, than is some other figure.

Using subgraph isomorphism, alike constituent geometric figures of the
original geometric figures can be found and checked for similarity.  A
simple similarity function can return a boolean value of 1 for similar
geometric figures and 0 otherwise.  However, such a function would
have limited applications.  In this paper, we define instead a
distance function that returns a value between 0 (inclusive) and
1. The returned value reflects the dissimilarity between alike planar
geometric figures connected with straight lines.

Therefore, the distance function $d$ is defined only when the graphs
representing the given geometric figures are
isomorphic~\cite{Ullmann1976}.  The crux of the function is in the
measurement of deviations from angular similarity and edge length
proportionality.

The function $d$ is the convex sum of functions $\alpha$ and $\rho$:

\begin{itemize}

\item The function $\alpha$, which we may call \emph{angular
  dissimilarity}, measures the deviation from the angular identity
  between two geometric figures.  In order to compute this, angles
  are projected on a cartesian plane, where the angles of the first
  geometric figure makes up one axis and the angles of the second
  geometric figure makes up the other axis.  Therefore, a cluster of
  points in this plane represents corresponding angles of the given
  geometric figures.  If the figures are similar (identical up to
  scale), the angular similarity property may be said to be satisfied,
  and the corresponding angle points lie on the $y=x$ line, and the
  value returned by $\alpha$ is zero.  If not, then the deviation from
  the property is now computed as the distance from the original point
  to the corresponding point on the $y=x$ line.

\item The function $\rho$, which we may call \emph{edge-length
  disproportionality}, measures the deviation from edge-length
  proportionality between geometric figures.  In order to compute
  this, the edge lengths are similarly projected to a cartesian plane,
  where the edge lengths of the first geometric figure makes up one
  axis and the edge lengths of the second geometric figure makes up
  the other axis.  The corresponding edge lengths of the given
  geometric figures are represented as points in this plane.  If two
  figures are proportional (identical up to scale), all corresponding
  edge-lengths are in a fixed proportion $m$, all points pass through
  a line $y = mx$, and the value returned by $\rho$ is zero.  In case
  the edge-lengths are not perfectly proportional, the calculation of
  $\rho$ comes to finding the best-fit line passing through the
  origin, and measuring the deviation from that line.

\end{itemize}

The choice of method to find the best fit line needs to consider the
fact that the line should pass through the origin.  Using the
least-squares method of fitting~\cite{Grewal2007} by adding $(0,0)$ as
one of the corresponding edge-length pairs does not give a proper line
passing through the origin.  This is the reason that the Iterative
Proportional Fitting Procedure (IPFP)~\cite{Wong1992} is used instead.
IPFP tries to find a fixed proportion among a set of pairs, thereby
giving points on the line passing through origin.

There are many IPFP~\cite{Lahr2004}, of which the one used in this
paper is the classical IPFP~\cite{Deming1940}, owing to its
simplicity.  On obtaining the required points from IPFP, the ratio
between any two points gives the values of $m$, as IPFP creates a
fixed proportionality among a set of edge-length pairs.
~\ref{appendix:Example4} explains step-by-step the IPFP technique used
in this paper.  Further, to compute the deviation from the edge-length
proportionality, we calculate the Euclidean distance between the
original point and the corresponding point on the line $y=mx$.  Sum up
the Euclidean distances of all edge-length pairs.  $\rho$ is computed
using this sum and a scaling factor.  As the considered geometric
figures are alike, the scaling factor is the number of edges in any
one of these geometric figures.  The need for this scaling factor
arises to account for the fact that in a large figure, with a large
number of edges, a minor change is less significant in determining
overall dissimilarity, than a corresponding change in a smaller
figure.

The function $d$ is shown to be a distance function as it satisfies
the three properties~\cite{Rudin1976} required: $d$ satisfies the
commutativity (Theorem~\ref{thm:dCommutative}) and triangular
inequality properties (Theorem~\ref{thm:dIneq}) defined over single
geometric figures.  However, the coincidence axiom is defined over
equivalence classes of geometric figures (figures that are alike up to
scale).  The proofs for these properties are given later in this
paper.

\section{The Distance Function}
\label{sec:DistanceFunction}
The distance function, represented by $d$, reflects the degree of
dissimilarity between figures.
 
Let, $\Gamma$ be the set of straight edge figures for which the
distance function is defined then
$$\gamma_i = (V_i, E_i, L_i, \Theta_i) \in \Gamma$$ where $V_i$
denotes the set of vertices, $E_i$ is the set of edges, $L_i$
represents the set of corresponding edge lengths and $\Theta_i$
denotes the set of angles that are defined between adjacent edges in
terms of radian.

Further, if $\gamma_i$ and $\gamma_j$ are said to be
  ``similar'', then $\gamma_i$ and $\gamma_j$ satisfy the
below conditions:

\begin{enumerate}

\item \label{itm:first} If $g_i=(V_i, E_i)$ is a graph that represents
  the adjacency of figure $\gamma_i$ and $g_j=(V_j, E_j)$ is a graph
  that represents the adjacency of figure $\gamma_j$, then graphs
  $g_i$ and $g_j$ are isomorphic.

\item \label{itm:second} All the corresponding angles of $\gamma_i$
  and $\gamma_j$ are equal, i.e., \newline 
if $\Theta_i = \{\theta_i(1), \theta_i(2), \ldots, \theta_i(z)\}$ 
represent the set of angles of $\gamma_i$ and \newline 
if $\Theta_j = \{\theta_j(1),
  \theta_j(2), \ldots, \theta_j(z)\}$ represent the set of
  corresponding angles of $\gamma_j$, then
\begin{equation}\label{angle_eq}
\theta_i(1) = \theta_j(1), \theta_i(2) = \theta_j(2), \ldots,
\theta_i(z) = \theta_j(z)
\end{equation}.

\item \label{itm:third} All the corresponding edge lengths of
  $\gamma_i$ and $\gamma_j$ are proportional, i.e., \newline 
if $L_i = \{l_i(1), l_i(2), \ldots, l_i(n)\}$ represent the set of edge
  lengths of $\gamma_i$ and \newline 
if $L_j = \{l_j(1), l_j(2), \ldots, l_j(n)\}$ represent the set of 
corresponding edge lengths of $\gamma_j$, then
\begin{equation} \label{prop_eq}
\frac{l_j(1)}{ l_i(1)} = \frac{l_j(2)}{l_i(2)} = \ldots = \frac{l_j(z)}{l_i(z)} = m \text{, a constant.}
\end{equation}.

\end{enumerate}

In view of this, the distance function tries to find the extent to which
the considered figures deviate from conditions \ref{itm:second} and
\ref{itm:third}, provided condition \ref{itm:first} is satisfied.

\begin{remark}
A few properties of the $d$ function:
\begin{enumerate}
\item $d: \Gamma \times \Gamma \to [0,1)$
\item $d(\gamma_i, \gamma_i) = 0$
\item $d(\gamma_i, \gamma_j) = 0$, if and only if $\gamma_i \approx
  \gamma_j$ \newline where $\approx$ denotes that $\gamma_i$ and
  $\gamma_j$ belong to same equivalence class of figures, \newline
  i.e., are figures that are identical up to scale.
\item $d$ satisfies the following:
\begin{equation}\label{sim_eq}
d(\gamma_i, \gamma_j) = 
\begin{cases}
0 & \text{if } \gamma_i \approx \gamma_j,\\
\lambda \in (0,1) & \text{otherwise}.
\end{cases}
\end{equation}
\end{enumerate}
\end{remark}

\section{Components of the Distance Function}
\label{sec:Components}
\subsection{Angular Dissimilarity}
\label{subsec:AngularDissimilarity}
Let $\alpha$ represent the angular dissimilarity function. Then the
function is defined as:
\begin{subequations}
\begin{equation}
\alpha: \Gamma \times \Gamma \rightarrow [0,1)
\end{equation}
\begin{equation} 
\alpha(\gamma_i, \gamma_j) = 
\begin{cases}
\ndownarrow & \text{if } \delta(g_i, g_j) = 0,\\
\varphi \in (0,1] & \text{otherwise}.
\end{cases}
\end{equation}
\end{subequations}

where $\delta$ represents the graph isomorphism function.
\begin{equation*}
\delta: G \times G \to \{0,1\}
\end{equation*}

with $G = \{g_1, g_2,\ldots\}$  : set of all graphs.
\begin{equation} \label{iso_eq}
\delta(g_i, g_j) = 
\begin{cases}
1 & \text{if } g_i \approx g_j,\\
0 & \text{otherwise}.
\end{cases}
\end{equation}
In \eqref{iso_eq}, the symbol $\approx$ denotes that $g_i$
and $g_j$ satisfy all properties of graph isomorphism.

Assuming $\delta(g_1,g_2) = 1, \alpha(\gamma_i, \gamma_j)$ is computed
as follows:
		
Project each corresponding pair $(\theta_i(u),\theta_j(u))$ into a
cartesian plane, wherein the $x$-axis represents the set $\Theta_i$,
while the $y$-axis represents the set $\Theta_j$.  The function
$\alpha$ computes the deviation from \eqref{angle_eq}.  In this
cartesian plane, according to \eqref{angle_eq}, all corresponding
pairs must lie on the line:
\begin{equation} \label{angleSimLineEqn}
y = x
\end{equation}

For each point $(\theta_i(u),\theta_j(u))$, calculate the Euclidean distance from
its corresponding point on the line \eqref{angleSimLineEqn}, i.e.,
$(\theta_i(u),\theta_i(u))$.
\begin{align}
\Lambda_{i,j}(u) &= \sqrt{(\theta_i(u) - \theta_i(u))^2 + (\theta_j(u) - \theta_i(u))^2} \nonumber \\
				&= \sqrt{(\theta_j(u) - \theta_i(u))^2} \nonumber\\
				&= |\theta_j(u) - \theta_i(u)| \label{angularShiftEqn}
\end{align}

Therefore, 
\begin{equation} \label{alphaEqn}
\alpha(\gamma_i, \gamma_j) = \frac{\sum_{u=1}^n \Lambda_{i,j}(u)} {1 +
  \sum_{u=1}^n \Lambda_{i,j}(u)}
\end{equation}

\begin{remark}
A few properties of the $\alpha$ function:
\begin{enumerate}
\item $\alpha(\gamma_i, \gamma_j) \ge 0$
\item $\alpha(\gamma_i, \gamma_i) = 0$
\item $\alpha(\gamma_i, \gamma_j)$ (which can be equal to 0), where $i \neq j$
\end{enumerate}
\end{remark}

\begin{theorem}\label{thm:alphaCommutative}
$\alpha(\gamma_i, \gamma_j) =  \alpha(\gamma_j, \gamma_i) $
\end{theorem}

\begin{proof}
We see that the constituents of the $\alpha$ function are commutative:
\begin{flalign*}
\qquad & \Lambda_{i,j}(u)= |\theta_j(u) - \theta_i(u)| & \\
&\Lambda_{i,j}(u) = \Lambda_{j,i}(u) \text{ , as} |a-b| = |b-a|&
\end{flalign*}

This follows that $\sum_{u=1}^n \Lambda_{i,j}(u) = \sum_{u=1}^n \Lambda_{j,i}(u) = e$, a constant

Hence,
\begin{align*}
\alpha(\gamma_i, \gamma_j) &=\frac{\sum_{u=1}^n \Lambda_{i,j}(u)} {1 +
  \sum_{u=1}^n \Lambda_{i,j}(u)} \\
 &= \frac{e}{1+e} \\
 &= \frac{\sum_{u=1}^n \Lambda_{j,i}(u)} {1 +
  \sum_{u=1}^n \Lambda_{j,i}(u)}\\
  &= \alpha(\gamma_j, \gamma_i) \hfill \qedhere
\end{align*}
\end{proof}

\begin{theorem}\label{thm:alphaIneq}
$\alpha(\gamma_i, \gamma_k) \le  \alpha(\gamma_i, \gamma_j) + \alpha(\gamma_j, \gamma_k)  $
\end{theorem}
\begin{proof}

\begin{flalign*}
\qquad & \Lambda_{i,k}(u) = |\theta_k(u) - \theta_i(u)| & \\
& \Lambda_{i,k}(u) \le \Lambda_{i,j}(u) + \Lambda_{j,k}(u) & \because |c-a| \le |b-a| + |c-b|
\end{flalign*}

Summing the above inequality for $p = 1 \text{ to } n $, it follows that
\begin{flalign} \label{lambdaIneqEqn}
\qquad & \sum_{u=1}^n \Lambda_{i,k}(u) \le \sum_{u=1}^n \Lambda_{i,j}(u) + \sum_{u=1}^n \Lambda_{j,k}(u) &
\end{flalign}

Let $ e, f \text{ and } g \text{ represent } \sum_{u=1}^n
\Lambda_{i,k}(u) , \sum_{u=1}^n \Lambda_{i,j}(u) \text{ and }
\sum_{u=1}^n \Lambda_{j,k}(u)$ respectively.

The inequality \eqref{lambdaIneqEqn} now translates to
\begin{flalign} \label{SimplelambdaIneqEqn}
\qquad & e \le f + g &
\end{flalign}

Assume that the contradiction of Theorem~\ref{thm:alphaIneq} is true, i.e, 
\begin{flalign} 
\qquad & \alpha(\gamma_i, \gamma_k) >  \alpha(\gamma_i, \gamma_j) + \alpha(\gamma_j, \gamma_k) \label{TheormContrEqn} &\\
 & \frac{e}{1+e} > \frac{f}{1+f} + \frac{g}{1+g} \nonumber&
\end{flalign}

On simplification, 
\begin{flalign}\label{contradictEqn}
\qquad & e > (f+g)  + (2fg + efg)  &
\end{flalign}

As the quantity $(2fg + efg) > 0$, the inequality
\eqref{contradictEqn} contradicts already proved inequality
\eqref{SimplelambdaIneqEqn}.  Hence, \eqref{TheormContrEqn} does not
hold true, thereby proving Theorem~\ref{thm:alphaIneq}.
\end{proof}

\subsection{Edge-Length Disproportionality}
\label{subsec:Edge-LengthDisproportionality}
Let $\rho$ represent the edge-length disproportionality function.  Then
the function is defined as:
\begin{subequations}
\begin{equation}
\rho: \Gamma \times \Gamma \rightarrow [0,1)
\end{equation}
\begin{equation} 
\rho(\gamma_i, \gamma_j) = 
\begin{cases}
\ndownarrow & \text{if } \delta(g_i, g_j) = 0,\\
\tau \in (0,1] & \text{otherwise}.
\end{cases}
\end{equation}
\end{subequations}
		
Assuming $\delta(g_1,g_2) = 1, \rho(\gamma_i, \gamma_j)$ is computed
as follows:

Project each corresponding pair $(l_i(h), l_j(h))$ into a cartesian
plane, wherein the $x$-axis represents the set $L_i$, while the
$y$-axis represents the set $L_j$.  The function $\rho$ computes the
deviation from \eqref{prop_eq}. Consider a part of the same equation.
\begin{equation} \label{slope_eq}
\frac{l_j(h)}{l_i(h)} = m \text{, a constant}
\end{equation}
		
In the context of the $L_iL_j$ plane, \eqref{slope_eq} gives the
slope of a line that passes through $(0,0) \text{and} (l_i(h),
l_j(h))$.
\begin{align*}
\text{Slope of a line, } m & = \frac{(y_2 - y_1)}{(x_2 - x_1)}\\
& = \frac{(l_j(h) - 0)}{(l_i(h) - 0)}\\
& = \eqref{slope_eq}
\end{align*}

Further extending this concept, it can be seen that in order to
satisfy \eqref{prop_eq} all points $(l_i(h), l_j(h))$ should
lie on the same line.  Therefore, finding edge length proportionality
now boils down to finding for the set of corresponding edge-length
pairs the best fit line, which passes though origin.

Let the equation of the required line be:
\begin{equation}\label{edge_line_eq}
y = mx  \text{, as the line passes through origin.}
\end{equation}
Using IPFP each point
$(l_i(h), l_j(h))$ is transformed to $(l_i'(h), l_j'(h))$, which is a
point on the line \ref{edge_line_eq}.

On finding the desired line, the euclidean distance between $(l_i(h),
l_j(h))$ and $(l_i'(h), l_j'(h))$ is computed.
\begin{equation} \label{euclideanEqn}
 \Delta_{i,j}(h) = \sqrt{(l_i(h) - l_i'(h))^2 + (l_j(h) - l_j'(h))^2}
\end{equation}

Therefore, 
\begin{equation} \label{rhoEqn}
\rho(\gamma_i, \gamma_j) = \frac {\sum_{h=1}^n \Delta_{i,j}(h)} {n + \sum_{h=1}^n \Delta_{i,j}(h)}
\end{equation}

\begin{remark}
A few properties of the $\rho$ function:
\begin{enumerate}
\item $\rho(\gamma_i, \gamma_j) \ge 0$
\item $\rho(\gamma_i, \gamma_i) = 0$
\item $\rho(\gamma_i, \gamma_j) \text{, where } i \neq j$ can be equal to $0$.
\end{enumerate}
\end{remark}

\begin{theorem} \label{thm:rhoCommutative}
$\rho(\gamma_i, \gamma_j) =  \rho(\gamma_j, \gamma_i) $
\end{theorem}

\begin{proof}
The proof is similar to that of Theorem \ref{thm:alphaCommutative}.
\end{proof}

\begin{theorem}\label{thm:rhoIneq}
$\alpha(\gamma_i, \gamma_k) \le  \alpha(\gamma_i, \gamma_j) + \alpha(\gamma_j, \gamma_k)  $
\end{theorem}
\begin{proof}
The proof is similar to that of Theorem \ref{thm:alphaIneq}.
\end{proof}

\subsection{Deriving the Function}
\label{subsec:DerivingFunction}
The function $d(\gamma_i, \gamma_j)$ is the convex sum of $\alpha(\gamma_i, \gamma_j)$ and $\rho(\gamma_i, \gamma_j)$.
\begin{subequations}
\begin{equation}
d : \Gamma \times \Gamma \nrightarrow [0,1)
\end{equation}
\begin{equation} \label{dDefinitionEqn}
d(\gamma_i, \gamma_j) = 
\begin{cases}
\ndownarrow & \text{if } \delta(g_i, g_j) = 0,\\
\beta \alpha(\gamma_i, \gamma_j) + (1-\beta) \rho(\gamma_i, \gamma_j) \text{, where } \beta \in [0,1] & \text{otherwise.}
\end{cases}
\end{equation}
\end{subequations}

While computing $d$ using ~\eqref{dDefinitionEqn} in ~\ref{appendix:Example1}, ~\ref{appendix:Example2}, ~\ref{appendix:Example3} and ~\ref{appendix:Example4}, the value of $\beta$ is set to $0.5$, to equally weight the $\alpha$ and $\rho$ functions. However, other values of $\beta \in [0,1]$ can be used resulting in similar outcomes for the $d$ function.

\begin{theorem} \label{thm:dCommutative}
$d(\gamma_i, \gamma_j) =  d(\gamma_j, \gamma_i) $
\end{theorem}

\begin{proof}

According to \eqref{dDefinitionEqn},
\begin{align*}
d(\gamma_i, \gamma_j) &= \beta \alpha(\gamma_i, \gamma_j) + (1-\beta) \rho(\gamma_j, \gamma_i), & \text{where } \beta \in [0, 1]\\
&= \beta \alpha(\gamma_j, \gamma_i) + (1-\beta) \rho(\gamma_i, \gamma_j), & \text{from Theorem \ref{thm:alphaCommutative} and Theorem \ref{thm:rhoCommutative}} \\
&= d(\gamma_j, \gamma_i). \hfill \qedhere
\end{align*}
\end{proof}

\begin{theorem}\label{thm:dIneq}
$d(\gamma_i, \gamma_k) \le  d(\gamma_i, \gamma_j) + d(\gamma_j, \gamma_k)  $
\end{theorem}

\begin{proof}
According to~\eqref{dDefinitionEqn}, $\beta \in [0, 1]$

Multiplying by $\beta$ both sides of the inequality in
Theorem~\ref{thm:alphaIneq}, we get:
\begin{flalign}
\qquad & \beta \alpha(\gamma_i, \gamma_k) \le  \beta \alpha(\gamma_i, \gamma_j) + \beta \alpha(\gamma_j, \gamma_k) & \label{BetaAlphaIneq}
\end{flalign}

Multiplying by $(1-\beta)$ both sides of the inequality in
Theorem~\ref{thm:rhoIneq}, we get:
\begin{flalign}
\qquad & (1-\beta)  \rho(\gamma_i, \gamma_k) \le  (1-\beta) \rho(\gamma_i, \gamma_j) + (1-\beta) \rho(\gamma_j, \gamma_k) & \label{BetaRhoIneq}
\end{flalign}
 
Summing up inequalities \eqref{BetaAlphaIneq} and \eqref{BetaRhoIneq}, it
follows that:
\begin{multline}
\beta \alpha(\gamma_i, \gamma_k) + (1-\beta)  \rho(\gamma_i, \gamma_k) \\ \le \beta \alpha(\gamma_i, \gamma_j) + (1-\beta) \rho(\gamma_i, \gamma_j) + \beta \alpha(\gamma_j, \gamma_k) + (1-\beta) \rho(\gamma_j, \gamma_k)
\end{multline}
\(d(\gamma_i, \gamma_k) \le d(\gamma_i, \gamma_j) + d(\gamma_j, \gamma_k)\). \qedhere
\end{proof}

\section{Results}
Using the above discussed method to compute the distance function,
$d$, this section tabulates the results for a few pairs of figures. It
can be found that the values of $d$ in Table~\ref{results} are
reflective of the dissimilarity of considered figures. The same can be
said for $\alpha$ and $\rho$ values.

\begin{table}[h]
\centering
\begin{tabular}{c c  >{$}l<{$} >{$}l<{$} >{$}l<{$}}
\toprule
\multicolumn{2}{c}{Figures to be compared} & \alpha & \rho & d \\
\midrule
\ref{fig:hex} & \ref{fig:skewHex} & 0.8073 &  0.4689 & 0.6381 \\
\midrule
\ref{fig:paraAndTri1} & \ref{fig:paraAndTri2} & 0.8073 &  0.7883 & 0.7978 \\
\midrule
\ref{fig:hexAndTri1} & \ref{fig:hexAndTri2} & 0.9281 & 0.9074 & 0.9177 \\
\midrule
\ref{fig:octAndQuad1} & \ref{fig:octAndQuad2} & 0.9201 & 0.7177 & 0.8189 \\
\bottomrule
\end{tabular}
\caption{Results obtained for a few pairs of figures.  See Appendices~\ref{appendix:Example1}, \ref{appendix:Example2}, \ref{appendix:Example3} and \ref{appendix:Example4} for more details regarding the tabulated results.}
\label{results}
\end{table}
\FloatBarrier


\bibliographystyle{amsplain} 
\bibliography{reference,new}

\providecommand{\bysame}{\leavevmode\hbox to3em{\hrulefill}\thinspace}
\providecommand{\MR}{\relax\ifhmode\unskip\space\fi MR }
\providecommand{\MRhref}[2]{%
  \href{http://www.ams.org/mathscinet-getitem?mr=#1}{#2}
}
\providecommand{\href}[2]{#2}
\begin{thebibliography}{10}

\bibitem{Blough2001}
Donald~S. Blough, \emph{The perception of similarity}, Avian Visual Cognition
  (Robert~G. Cook, ed.), September 2001.

\bibitem{Bourne2009}
P.~E. Bourne and J.~Gu, \emph{Structural bioinformatics}, 2 ed., Wiley, 2009.

\bibitem{Deming1940}
W.~Edwards Deming and Frederick~F. Stephan, \emph{{On a Least Squares
  Adjustment of a Sampled Frequency Table When the Expected Marginal Totals are
  Known}}, The Annals of Mathematical Statistics \textbf{11} (1940), no.~4,
  427--444.

\bibitem{Gillespie2005}
Ronald~J. Gillespie and Edward~A. Robinson, \emph{Models of molecular
  geometry}, Chemical Society Reviews \textbf{34} (2005), 396--407.

\bibitem{Grewal2007}
B.~S. Grewal and J.~S. Grewal, \emph{{Higher Engineering Mathematics}}, 40 ed.,
  Khanna Publishers, New Delhi, 2007.

\bibitem{Heller2011}
Valentin Heller, \emph{Scale effects in physical hydraulic engineering models},
  ournal of Hydraulic Research \textbf{49} (2011), no.~3, 293--306.

\bibitem{Komosinski2001}
Maciej Komosinski, Grzegorz Koczyk, and Marek Kubiak, \emph{On estimating
  similarity in artificial and real organisms}, Theory in Biosciences
  \textbf{120} (2001), no.~3--4, 271--286.

\bibitem{Komosinski2011}
Maciej Komosinski and Marek Kubiak, \emph{Quantitative measure of structural
  and geometric similarity of {3D} morphologies}, Complexity \textbf{16}
  (2011), no.~6, 40--52.

\bibitem{Lahr2004}
Michael Lahr and Louis de~Mesnard, \emph{{Biproportional Techniques in
  Input-Output Analysis: Table Updating and Structural Analysis}}, Economic
  Systems Research \textbf{16} (2004), no.~2, 115--134.

\bibitem{Pallett1961}
G.~Pallett, \emph{Geometric similarity---some applications in fluid mechanics},
  Education + Training \textbf{3} (1961), no.~2, 36--37.

\bibitem{Rudin1976}
Walter Rudin, \emph{{Principles of Mathematical Analysis}}, McGrawHill Inc.,
  New York, 1976.

\bibitem{Sen1987}
Ayusman Sen, Venkatasuryanarayana Chebolu, and Arnold~L. Rheingold, \emph{First
  structurally characterized geometric isomers of an eight-coordinate complex.
  structural comparison between cis- and
  trans-diiodobis(2,5,8-trioxanonane)samarium}, Inorganic Chemistry \textbf{26}
  (1987), no.~11, 1821--1823.

\bibitem{Stashans2008}
A.~Stashans, G.~Chamba, and H.~Pinto, \emph{Electronic structure, chemical
  bonding, and geometry of pure and {Sr}-doped {CaCO3}}, Journal of
  Computational Chemistry \textbf{29} (2008), no.~3, 343--349.

\bibitem{Stoker2009}
H.~Stephen Stoker, \emph{General, organic, and biological chemistry}, Cengage
  Learning, 2009.

\bibitem{Sweeney2015}
Chris Sweeney, Laurent Kneip, Tobias H\"{o}llerer, and Matthew Turk,
  \emph{Computing similarity transformations from only image correspondences},
  The IEEE Conference on Computer Vision and Pattern Recognition (CVPR 2015),
  June 2015, doi:10.1109/CVPR.2015.7298951, pp.~3305--3313.

\bibitem{Tarr1998}
Michael~J. Tarr and Heinrich~H. B\"{u}lthoff, \emph{Image-based object
  recognition in man, monkey and machine}, Cognition \textbf{67} (1998),
  no.~1--2, 1--20, doi:10.1016/S0010-0277(98)00026-2.

\bibitem{Ullmann1976}
J.~R. Ullmann, \emph{{An Algorithm for Subgraph Isomorphism}}, Journal of the
  ACM \textbf{23} (1976), no.~1, 31--42.

\bibitem{Vermorken2008}
Maximilian Vermorken, Ariane Szafarz, and Hugues Pirotte, \emph{Sector
  classification through non-gaussian similarity}, Applied Financial Economics
  \textbf{20} (2008), no.~11, doi:10.1080/09603101003636238.

\bibitem{Wong1992}
David W.~S. Wong, \emph{{The Reliability of Using the Iterative Proportional
  Fitting Procedure∗}}, 1992, pp.~340--348.

\bibitem{Yu1996}
Shiuh-Sheng Yu, Jinn-Rong Liou, and Wen-Chin Shen, \emph{Computational
  similarity based on chromatic barycenter algorithm}, IEEE Transactions on
  Consumer Electronics \textbf{42} (1996), no.~2, 216--220.

\bibitem{phdthesis}
Bilal Zaka, \emph{Theory and applications of similarity detection techniques},
  Ph.D. thesis, Institute for Information Systems and Computer Media (IICM),
  Graz University of Technology, Graz, Austria, 2 2009.

\end{thebibliography}

\appendix

\section{Computing $d$ for Figures~\ref{fig:hex} and \ref{fig:skewHex}}
\label{appendix:Example1}
\begin{figure}
\centering
\begin{subfigure}{0.35\textwidth}
\includegraphics[width=4cm]{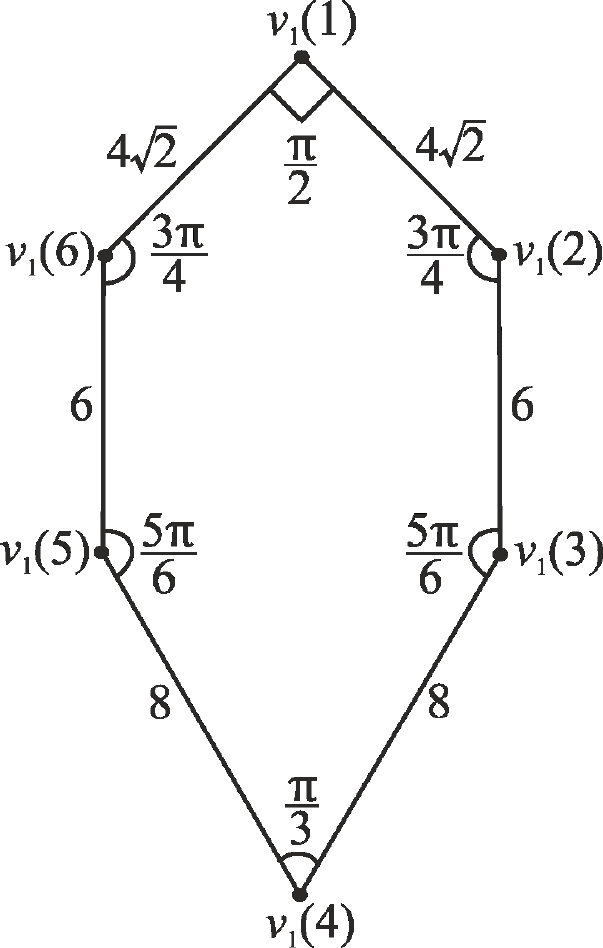}
\caption{$\bm{\gamma_1}$}
\label{fig:hex}
\end{subfigure}
\qquad 
\begin{subfigure}{0.55\textwidth}
\includegraphics[width=7.5cm]{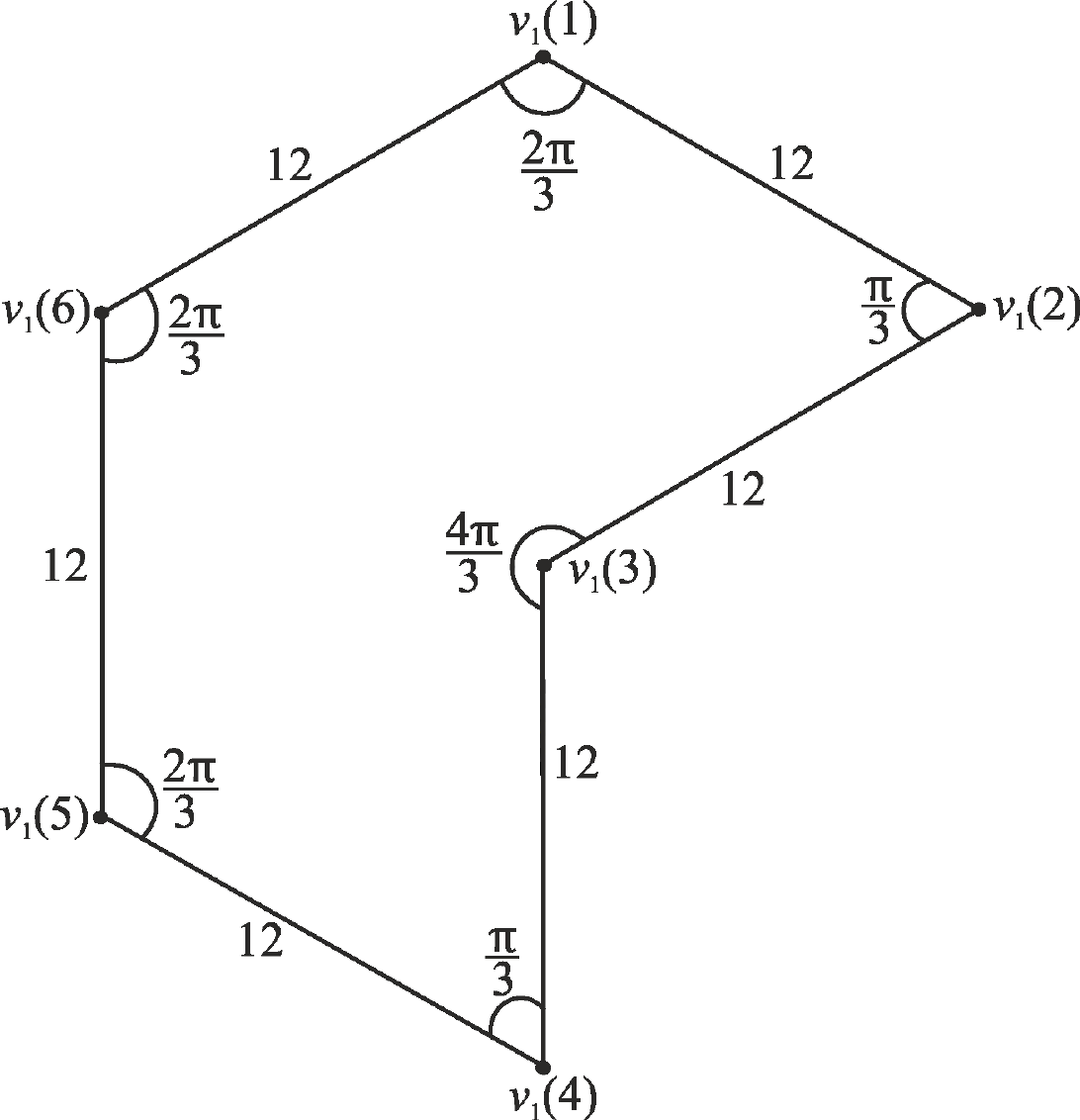}
\caption{$\bm{\gamma_2$}}
\label{fig:skewHex}
\end{subfigure}
\caption{}
\end{figure}

Computing $\alpha(\gamma_i, \gamma_j)$
\begin{flalign*}
\qquad & \Theta_1 = \left\{ \frac{\pi}{2}, \frac{3\pi}{4}, \frac{5\pi}{6}, \frac{\pi}{3}, \frac{5\pi}{6}, \frac{3\pi}{4} \right\} & \\
& \Theta_2 = \left\{ \frac{2\pi}{3}, \frac{\pi}{3}, \frac{4\pi}{3}, \frac{\pi}{3}, \frac{2\pi}{3}, \frac{2\pi}{3} \right\} & 
\end{flalign*}

Using \eqref{angularShiftEqn} we compute the Euclidean distance.
\begin{flalign*}
\qquad & \Lambda_{1,2}(1) = \frac{\pi}{6} \quad \Lambda_{1,2}(2) = \frac{5\pi}{12} \quad \Lambda_{1,2}(3) = \frac{\pi}{2} \quad \Lambda_{1,2}(4) = 0 \quad \Lambda_{1,2}(5) =\frac{\pi}{6} \quad \Lambda_{1,2}(6) = \frac{\pi}{12} & \\
& \sum_{u=1}^6 \Lambda_{1,2}(u) = \frac{4\pi}{3} & \\
& \bm{\alpha(\gamma_1,\gamma_2)} = 0.8073 \text{, using \eqref{alphaEqn}}&
\end{flalign*}

Figure~\ref{fig:Hex} shows the relation between the corresponding elements of $\Theta_1$ and $\Theta_2$. It also indicates, for each pair $\langle \theta_1(u), \theta_2(u) \rangle$ the corresponding point, $\langle \theta_1(u), \theta_1(u) \rangle$, on the $y=x$ line. Further, Table~\ref{tab:Hex} provides the legend for this figure.
\begin{figure}[h]
\centering
\includegraphics[width=0.8\textwidth]{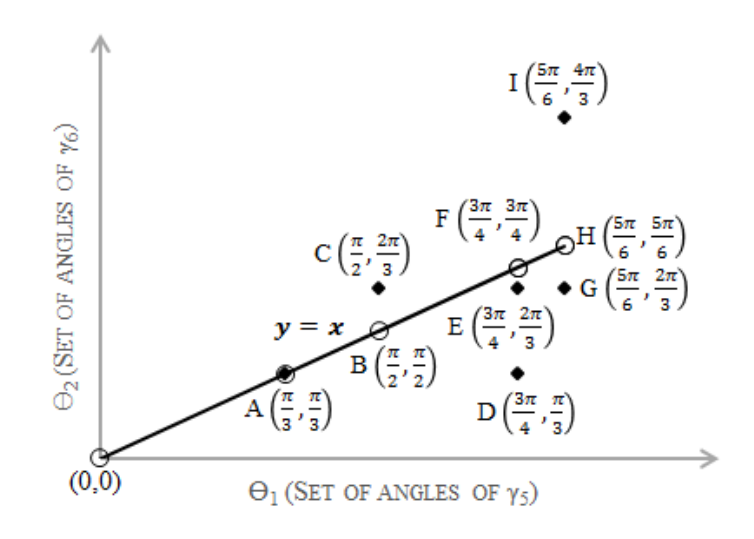}
\caption {\scriptsize Solid dots show the existing relation between $\Theta_1$ and $\Theta_2$, whereas the line passing through the hollow dots shows the expected relation between $\Theta_1$ and $\Theta_2$}
\label{fig:Hex}
\end{figure}
\FloatBarrier
\begin{table}[h]
\begin{tabular}{ >{$}l<{$} @{\;$\in$\;} >{$\{}l<{\}$} >{$}l<{$} @{\;$\in$\;} >{$\{}l<{\}$} }
A & \langle \theta_1(4), \theta_2(4) \rangle, \langle \theta_1(4), \theta_1(4) \rangle & B & \langle \theta_1(1), \theta_1(1) \rangle \\
C & \langle \theta_1(1), \theta_2(1) \rangle & D & \langle \theta_1(2), \theta_2(2) \rangle \\
E & \langle \theta_1(6), \theta_2(6) \rangle & F & \langle \theta_1(6), \theta_1(6) \rangle \\
G & \langle \theta_1(5), \theta_2(5) \rangle & H & \langle \theta_1(3), \theta_1(3) \rangle, \langle \theta_1(5), \theta_1(5) \rangle \\
I & \langle \theta_1(3), \theta_2(3) \rangle \\
\end{tabular}
\caption{Legend of Figure~\ref{fig:Hex}}
\label{tab:Hex}
\end{table}
\FloatBarrier

Computing $\rho(\gamma_i, \gamma_j)$
\begin{flalign*}
\qquad & L_1 = \{ 4\sqrt{2}, 4\sqrt{2}, 6, 8, 8, 6 \} & \\
& L_2 = \{ 12, 12, 12, 12, 12, 12 \} &
\end{flalign*}

Table~\ref{tab:Fig1And2} indicates the input and output of IPFP transformation. 

\begin{table}[h]
\resizebox{\textwidth}{!}{%
\begin{subtable}[h]{0.6\linewidth}\centering
{\begin{tabular}{llll}
\toprule
$h$ & $\bm{l_1(h)}$ & $\bm{l_2(h)}$ & {\sc total} \\ 
\midrule
1 & 5.6569 & 12 & 17.6569\\  
\midrule
2 & 5.6569 & 12 & 17.6569\\
\midrule
3 & 6 & 12 & 18\\
\midrule
4 & 8 & 12 & 20\\
\midrule
5 & 8 & 12 & 20\\
\midrule
6 & 6 & 12 & 18\\
\midrule
{\sc total} & \bf{39.3138} & \bf{72} & \bf{111.3138} \\ 
\bottomrule
\end{tabular}}
\caption{Values on which IPFP is to be performed}
\label{tab:hexTable1}
\end{subtable}
\begin{subtable}[h]{0.6\linewidth}\centering
{\begin{tabular}{llll}
\toprule
$h$ & $\bm{l_1'(h)}$ & $\bm{l_2'(h)}$ & {\sc total}\\ 
\midrule
1 & 6.236 & 11.4208 & 17.6569\\  
\midrule
2 & 6.236 & 11.4208 & 17.6569\\  
\midrule
3 & 6.3572 & 11.6428 & 18\\
\midrule
4 & 7.0636 & 12.9364 & 20\\
\midrule
5 & 7.0636 & 12.9364 & 20\\
\midrule
6 & 6.3572 & 11.6428 & 18\\
\midrule
{\sc total} & \bf{39.3138} & \bf{72} & \bf{111.3138} \\ 
\bottomrule
\end{tabular}}
\caption{Values obtained on applying IPFP}
\label{tab:hexTable2}
\end{subtable}
}
\caption{IPFP Transformation}
\label{tab:Fig1And2}
\end{table}

\begin{flalign*}
\qquad & \text{Considering any row from the Table~\ref{tab:hexTable2}, we compute $m$.} & \\
& m = \frac{11.4208}{6.236} =1.8314 &\\
& y =  1.8314 x \text{, equation of the expected line} &
\end{flalign*}

Figure~\ref{fig:Hex-edge} shows the relation between the corresponding
elements of $L_1$ and $L_2$. It also indicates, for each pair,
$\langle l_1(h), l_2(h) \rangle$, the corresponding point, $\langle
l_1'(h), l_2'(h) \rangle$, on line $y=1.8314 x$ line. Further,
Table~\ref{tab:Hex-edge} provides the legend for this figure.

\begin{figure}[h!]
\centering
\includegraphics[width=0.8\textwidth]{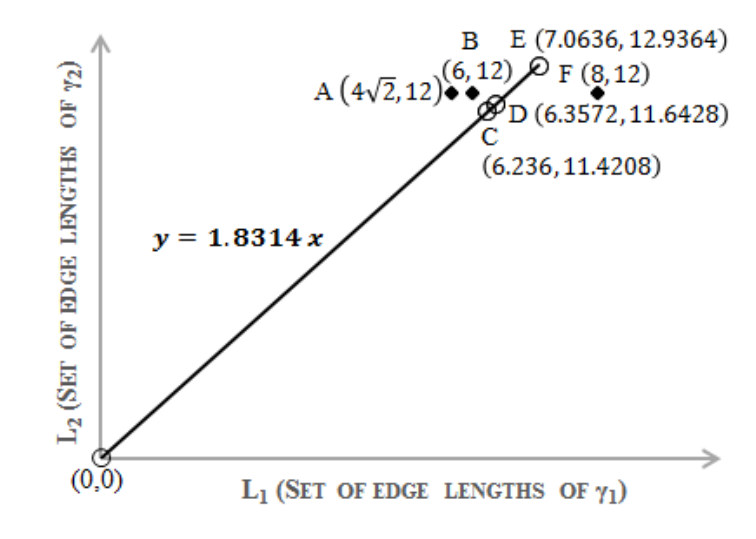}
\caption {\scriptsize Solid dots show the existing relation between $L_1$ and $L_2$, whereas the line passing through the hollow dots shows the expected relation between $L_1$ and $L_2$}
\label{fig:Hex-edge}
\end{figure}
\FloatBarrier
\begin{table}
\resizebox{\textwidth}{!}{%
\begin{tabular}{ >{$}l<{$} @{\;$\in$\;} >{$\{}l<{\}$} >{$}l<{$} @{\;$\in$\;} >{$\{}l<{\}$}}
A & \langle l_1(1), l_2(1) \rangle, \langle l_1(2), l_2(2) \rangle & B & \langle l_1(3), l_2(3) \rangle, \langle l_1(6), l_2(6) \rangle \\
C & \langle l_1'(1), l_2'(1) \rangle, \langle l_1'(2), l_2'(2) \rangle & D & \langle l_1'(3), l_2'(3) \rangle, \langle l_1'(6), l_2'(6) \rangle \\
E & \langle l_1'(4), l_2'(4) \rangle, \langle l_1'(5), l_2'(5) \rangle & F & \langle l_1(4), l_2(4) \rangle, \langle l_1(5), l_2(5) \rangle \\
\end{tabular}\newline
}
\caption{Legend of Figure~\ref{fig:Hex-edge}}
\label{tab:Hex-edge}
\end{table}
\FloatBarrier

Now, the Euclidean Distance is computed using \eqref{euclideanEqn} 

\begin{tabular}{ >{$}l<{$} >{$}l<{$} >{$}l<{$} }
\Delta_{1,2}(1) = 0.8191  & \Delta_{1,2}(2) = 0.8191  & \Delta_{1,2}(3) = 0.5052 \\ 
\Delta_{1,2}(4) = 1.3243  & \Delta_{1,2}(5) = 1.3243  & \Delta_{1,2}(6) = 0.5052  \\  
\end{tabular}
\begin{flalign*}
\qquad & \sum_{h=1}^6 \Delta_{1,2}(h) = 5.2971 & \\
& \bm{\rho(\gamma_1,\gamma_2)} =0.4689 \text{, using \eqref{rhoEqn}} &
\end{flalign*}

Computing $d(\gamma_i, \gamma_j)$
\begin{flalign*}
\qquad & \bm{d(\gamma_1,\gamma_2)} = 0.6381 \text{, with } \beta = 0.5 \text{, using \eqref{dDefinitionEqn}}& 
\end{flalign*}
\FloatBarrier

\section{Computing $d$ for Figures~\ref{fig:paraAndTri1} and \ref{fig:paraAndTri2}}
\label{appendix:Example2}
\begin{figure}[h]
\centering
\begin{subfigure}[b]{0.45\textwidth}
\includegraphics[width=8cm]{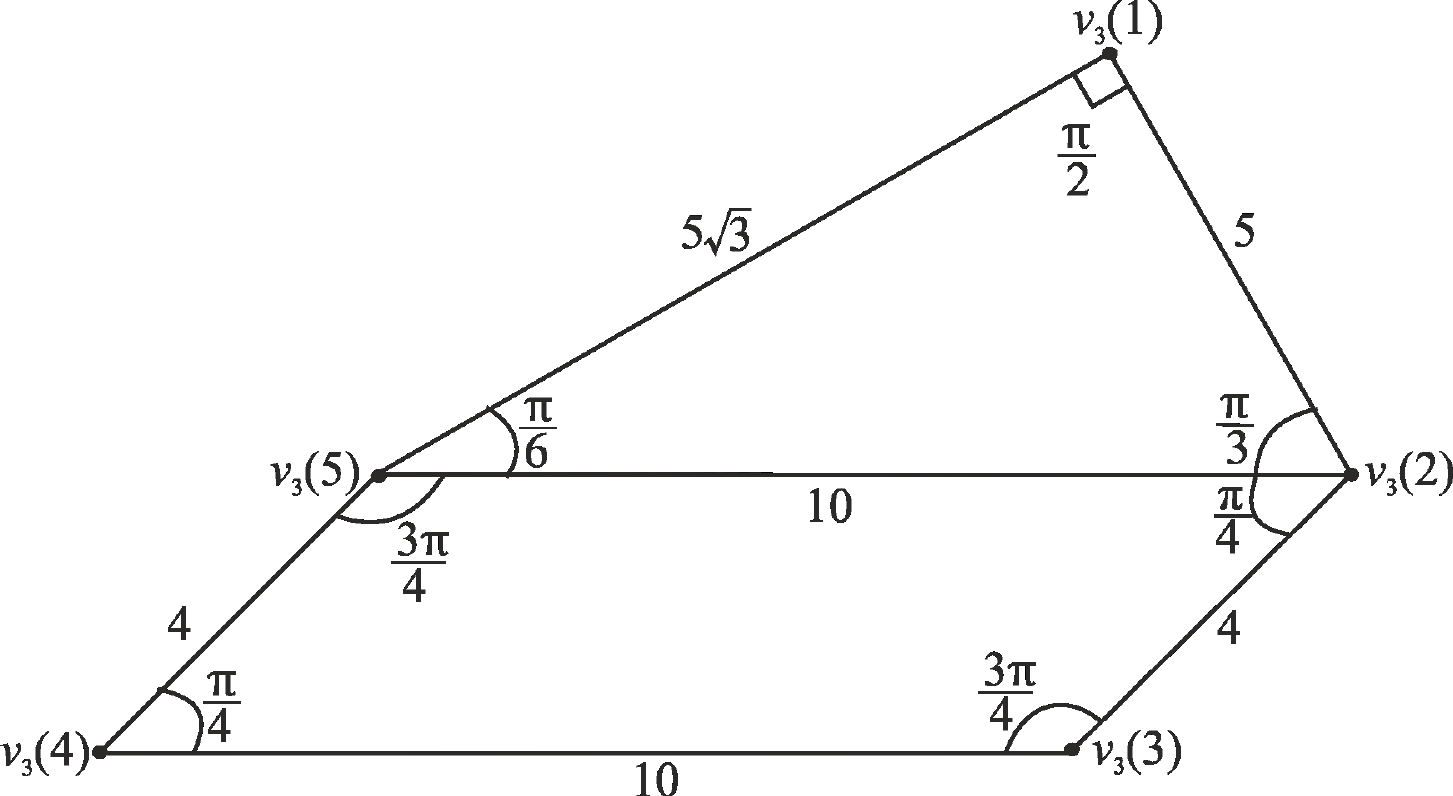}
\caption{$\bm{\gamma_3}$}
\label{fig:paraAndTri1}
\end{subfigure}
\hfill
\begin{subfigure}[b]{0.35\textwidth}
\includegraphics[width=4cm]{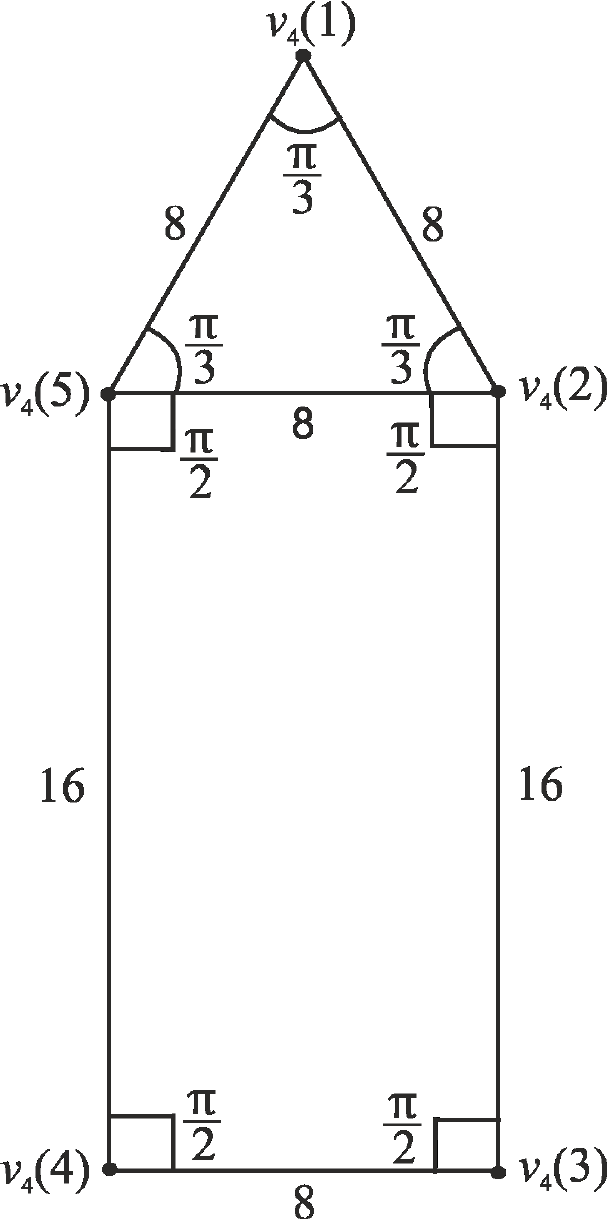}
\caption{$\bm{\gamma_4$}}
\label{fig:paraAndTri2}
\end{subfigure}
\caption{}
\end{figure}

Computing $\alpha(\gamma_3,\gamma_4)$
\begin{flalign*}
\qquad & \Theta_3 = \left\{ \frac{\pi}{2}, \frac{\pi}{3}, \frac{3\pi}{4}, \frac{\pi}{4}, \frac{3\pi}{4}, \frac{\pi}{6} \right\} & \\
& \Theta_4 = \left\{ \frac{\pi}{3}, \frac{\pi}{3}, \frac{\pi}{2}, \frac{\pi}{2}, \frac{\pi}{2}, \frac{\pi}{3} \right\} & \\
& \text{Using \eqref{angularShiftEqn} we compute the Euclidean distance.}& \\
& \Lambda_{3,4}(1) = \frac{\pi}{6} \quad \Lambda_{3,4}(2) = 0 \quad \Lambda_{3,4}(3) = \frac{\pi}{4}& \\
& \Lambda_{3,4}(4) =  \frac{\pi}{4} \quad \Lambda_{3,4}(5) = \frac{\pi}{4} \quad \Lambda_{3,4}(6) = \frac{\pi}{4}&\\ 
& \Lambda_{3,4}(7) = \frac{\pi}{6} &\\
& \sum_{u=1}^7 \Lambda_{3,4}(u) = \frac{4\pi}{3} &\\
& \bm{\alpha(\gamma_3,\gamma_4)} = 0.8073 \text{, using \eqref{alphaEqn}}&
\end{flalign*}

Figure~\ref{fig:Comp1} shows the relation between the corresponding
elements of $\Theta_3$ and $\Theta_4$. It also indicates, for each
pair $\langle \theta_3(u), \theta_4(u) \rangle$ the corresponding
point, $\langle \theta_3(u), \theta_3(u) \rangle$, on the $y=x$
line. Further, Table~\ref{tab:Comp1} provides the legend for this
figure.

\begin{figure}[h]
\centering
\includegraphics[width=0.6\textwidth]{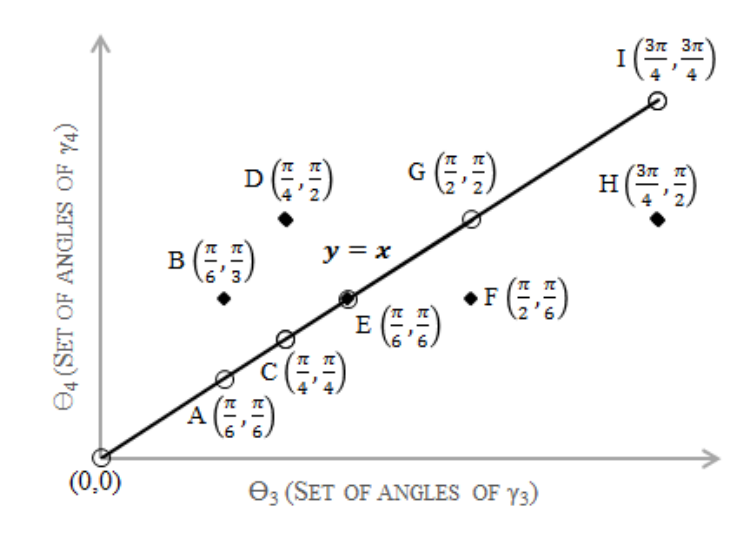}
\caption{\scriptsize Solid dots show the existing relation between $\Theta_3$ and $\Theta_4$, whereas the line passing through the hollow dots shows the expected relation between $\Theta_3$ and $\Theta_4$}
\label{fig:Comp1}
\end{figure}
\FloatBarrier
\begin{table}
\resizebox{\textwidth}{!}{%
\begin{tabular}{>{$}l<{$} @{\;$\in$\;} >{$\{}l<{\}$} >{$}l<{$} @{\;$\in$\;} >{$\{}l<{\}$} }
A & \langle \theta_3(7), \theta_3(7) \rangle & B & \langle \theta_3(7), \theta_4(7) \rangle \\
C & \langle \theta_3(3), \theta_3(3) \rangle, \langle \theta_3(5), \theta_3(5) \rangle & D & \langle \theta_3(3), \theta_4(3) \rangle, \langle \theta_3(5), \theta_4(5) \rangle \\
E & \langle \theta_3(2), \theta_4(2) \rangle, \langle \theta_3(2), \theta_3(2) \rangle & F & \langle \theta_3(1), \theta_4(1) \rangle \\
G & \langle \theta_3(1), \theta_3(1) \rangle & H & \langle \theta_3(4), \theta_4(4) \rangle, \langle \theta_3(6), \theta_4(6) \rangle \\
I & \langle \theta_3(4), \theta_3(4) \rangle, \langle \theta_3(6), \theta_3(6) \rangle \\
\end{tabular}
}
\caption{Legend of Figure~\ref{fig:Comp1}}
\label{tab:Comp1}
\end{table}
\FloatBarrier

Computing $\rho(\gamma_3,\gamma_4)$
\begin{flalign*}
\qquad &L_3 = \{5\sqrt(3), 5, 10, 4, 10, 4 \} & \\
&L_4 = \{ 8, 8, 8, 16, 8, 16 \} &
\end{flalign*}

Table~\ref{tab:Fig3And4} indicates the input and output of IPFP transformation. 

\begin{table}[h]
\resizebox{\textwidth}{!}{%
\begin{subtable}[h]{0.6\linewidth}
\centering
\begin{tabular}{llll}
\toprule
$h$ & $\bm{l_3(h)}$ & $\bm{l_4(h)}$ & {\sc total}\\ 
\midrule
1 & 8.6603 & 8 & 16.6603\\  
\midrule
2 & 5 & 8 & 13\\
\midrule
3 & 10 & 8 & 18\\
\midrule
4 & 4 & 16 & 20\\
\midrule
5 & 10 & 8 & 18\\
\midrule
6 & 4 & 16 & 20\\
\midrule
{\sc total} & \bf{41.6603} & \bf{64} & 105.6603 \\ 
\bottomrule
\end{tabular}
\caption{Values on which IPFP is to be performed}
\label{tab:Comp1Table1}
\end{subtable}%
\begin{subtable}[h]{0.6\linewidth}
\centering
\begin{tabular}{llll}
\toprule
$h$ & $\bm{l_3'(h)}$ & $\bm{l_4'(h)}$ & {\sc total}\\ 
\midrule
1 & 6.5689 & 10.0914 & 16.6603\\  
\midrule
2 & 5.1257 & 7.8743 & 13\\
\midrule
3 & 7.0971 & 10.9029 & 18\\
\midrule
4 & 7.8857 & 12.1143 & 20\\
\midrule
5 & 7.0971 & 10.9029 & 18\\
\midrule
6 & 7.8857 & 12.1143 & 20\\
\midrule
{\sc total} & \bf{41.6603} & \bf{64} & 105.6603 \\ 
\bottomrule
\end{tabular}
\caption{Values obtained on applying IPFP}
\label{tab:Comp1Table2}
\end{subtable}
}
\caption{IPFP Transformation}
\label{tab:Fig3And4}
\end{table}

\begin{flalign*}
\qquad & \text{Considering any row from the Table~\ref{tab:Comp1Table2}, we compute $m$} & \\
& m = \frac{10.0914}{6.5689} = 1.5362 &\\
& y =  1.5362 x \text{, equation of the expected line} &
\end{flalign*}

Figure~\ref{fig:Comp1-edge} shows the relation between the corresponding elements of $L_3$ and $L_4$. It also indicates, for each pair $\langle l_3(h), l_4(h) \rangle$ the corresponding point, $\langle l_3'(h), l_4'(h) \rangle$, on the $y=1.5362 x$ line. Further, Table~\ref{tab:Comp1-edge} provides the legend for this figure.

\begin{figure}[h]
\centering
\includegraphics[width=0.6\textwidth]{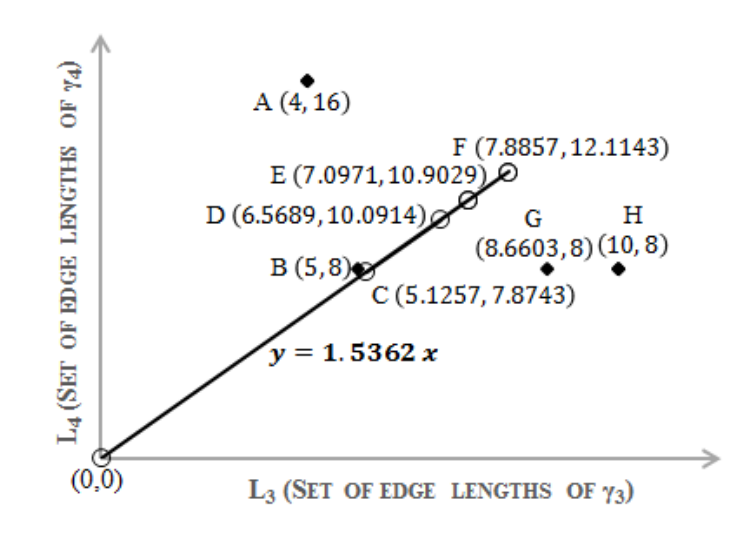}
\caption{\scriptsize Solid dots show the existing relation between $L_3$ and $L_4$, whereas the line passing through the hollow dots shows the expected relation between $L_3$ and $L_4$}
\label{fig:Comp1-edge}
\end{figure}
\FloatBarrier
\begin{table}
\begin{tabular}{>{$}l<{$} @{\;$\in$\;} >{$\{}l<{\}$} >{$}l<{$} @{\;$\in$\;} >{$\{}l<{\}$} }
A & \langle l_3(4), l_4(4) \rangle, \langle l_3(6), l_4(6) \rangle & B & \langle l_3(2), l_4(2) \rangle \\
C & \langle l_3'(2), l_4'(2) \rangle & D & \langle l_3'(1), l_4'(1) \rangle \\
E & \langle l_3'(3), l_4'(3) \rangle, \langle l_3'(5), l_4'(5) \rangle & F & \langle l_3'(4), l_4'(4) \rangle, \langle l_3'(6), l_4'(6) \rangle \\
G & \langle l_3(1), l_4(1) \rangle & H & \langle l_3(3), l_4(3) \rangle, \langle l_3(5), l_4(5) \rangle \\
\end{tabular}\newline
\caption{Legend of Figure~\ref{fig:Comp1-edge}}
\label{tab:Comp1-edge}
\end{table}
\FloatBarrier

Now, the Euclidean Distance is computed using \eqref{euclideanEqn} 

\begin{tabular}{ >{$}l<{$} >{$}l<{$} >{$}l<{$} }
\Delta_{3,4}(1) = 2.9577  & \Delta_{3,4}(2) = 0.1778  & \Delta_{3,4}(3) = 4.1053 \\ 
\Delta_{3,4}(4) = 5.4952  & \Delta_{3,4}(5) = 4.1053  & \Delta_{3,4}(6) = 5.4952  \\  
\end{tabular}
\begin{flalign*}
\qquad & \sum_{h=1}^6 \Delta_{3,4}(h) = 22.3365 & \\
& \bm{\rho(\gamma_3,\gamma_4)} = 0.7883 \text{, using \eqref{rhoEqn}}&
\end{flalign*}

Computing $d(\gamma_3,\gamma_4)$
\begin{flalign*}
\qquad & \bm{d(\gamma_3,\gamma_4)} = 0.7978 \text{, with } \beta = 0.5 \text{, using \eqref{dDefinitionEqn}}&
\end{flalign*}
\FloatBarrier

\section{Computing $d$ for Figures~\ref{fig:hexAndTri1} and \ref{fig:hexAndTri2}}
\label{appendix:Example3}
\begin{figure}[h]
\centering
\begin{subfigure}[b]{0.45\textwidth}
\includegraphics[width=8cm]{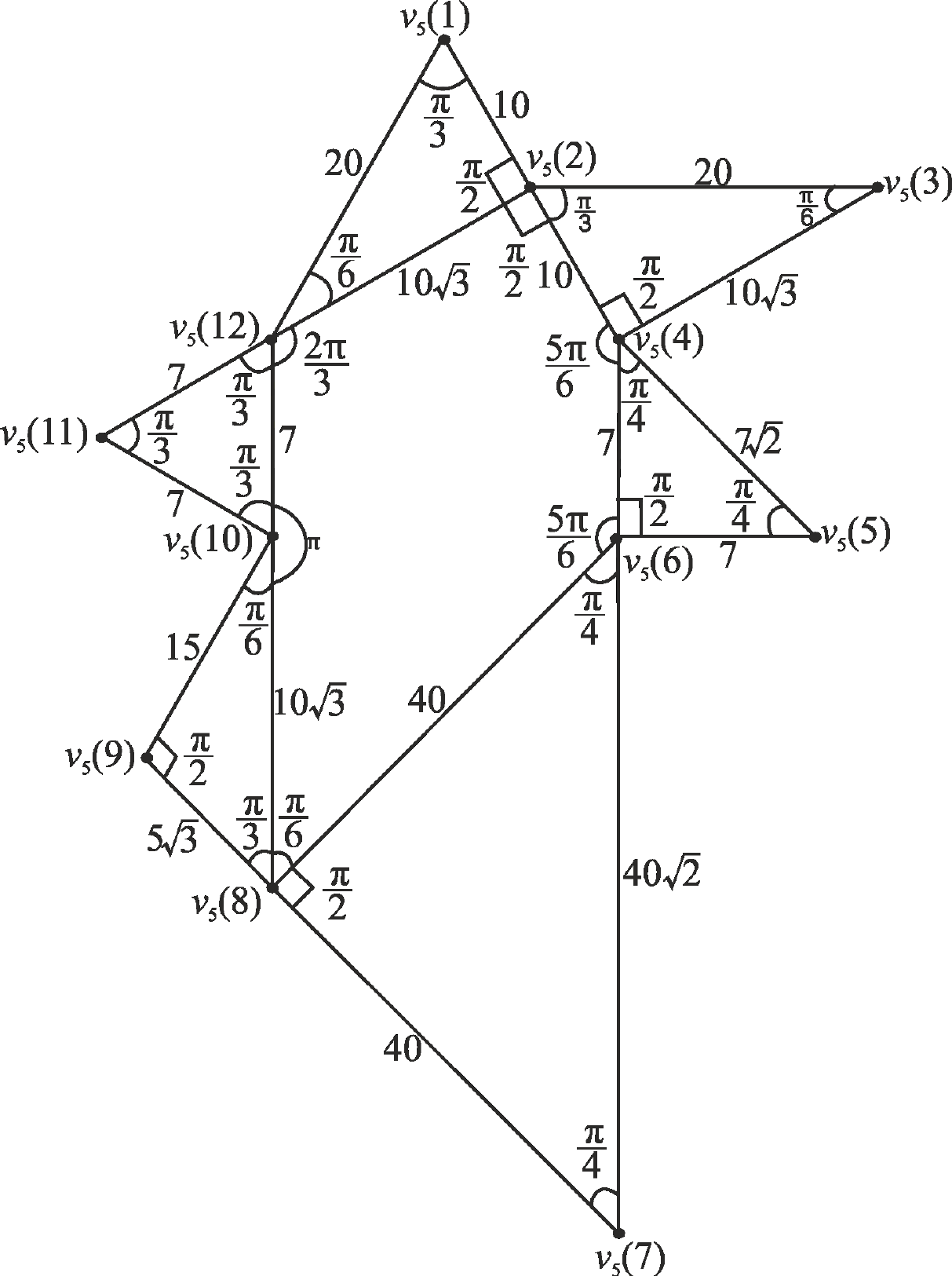}
\caption{$\bm{\gamma_5}$}
\label{fig:hexAndTri1}
\end{subfigure}
\hfill
\begin{subfigure}[b]{0.45\textwidth}
\includegraphics[width=8cm]{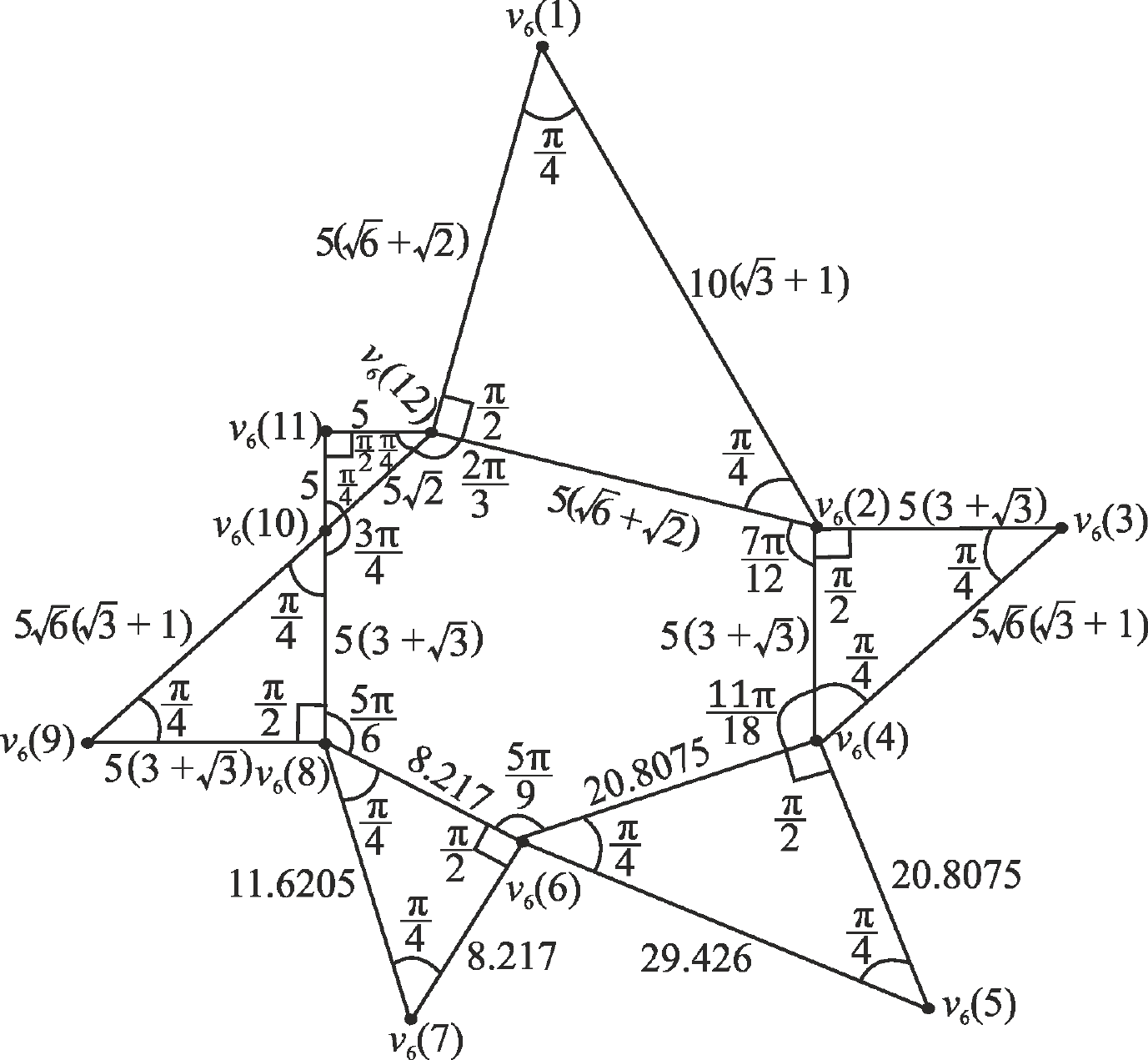}
\caption{$\bm{\gamma_6$}}
\label{fig:hexAndTri2}
\end{subfigure}
\caption{}
\end{figure}

Computing $\alpha(\gamma_5,\gamma_6)$, we get:

\begin{flalign*}
\tiny
\qquad & \Theta_5 = \left\{ \frac{\pi}{3}, \frac{\pi}{2}, \frac{\pi}{2}, \frac{\pi}{3}, \frac{\pi}{6}, \frac{\pi}{2}, \frac{5\pi}{6}, \frac{\pi}{4}, \frac{\pi}{4}, \frac{\pi}{2}, \frac{5\pi}{6}, \frac{\pi}{4}, \frac{\pi}{4}, \frac{\pi}{2}, \frac{\pi}{6}, \frac{\pi}{3}, \frac{\pi}{2}, \frac{\pi}{6}, \pi, \frac{\pi}{3}, \frac{\pi}{3}, \frac{\pi}{3}, \frac{2\pi}{3}, \frac{\pi}{6} \right\} & \\
& \Theta_6 = \left\{ \frac{\pi}{2}, \frac{\pi}{4}, \frac{2\pi}{3}, \frac{\pi}{2}, \frac{\pi}{4}, \frac{\pi}{4}, \frac{7\pi}{12}, \frac{\pi}{2}, \frac{\pi}{4}, \frac{\pi}{4}, \frac{11\pi}{18}, \frac{\pi}{2}, \frac{\pi}{4}, \frac{\pi}{4}, \frac{5\pi}{9}, \frac{\pi}{2}, \frac{\pi}{4}, \frac{\pi}{4}, \frac{5\pi}{6}, \frac{\pi}{2}, \frac{\pi}{4}, \frac{\pi}{4}, \frac{3\pi}{4}, \frac{\pi}{4} \right\} & 
\end{flalign*}

Using \eqref{angularShiftEqn} we compute the Euclidean distance.

\resizebox{\textwidth}{!}{%
\begin{tabular}{ >{$}l<{$} @{\;=\;} >{$}l<{$} >{$}l<{$} @{\;=\;} >{$}l<{$} >{$}l<{$} @{\;=\;} >{$}l<{$} >{$}l<{$} @{\;=\;} >{$}l<{$} >{$}l<{$} @{\;=\;} >{$}l<{$} >{$}l<{$} @{\;=\;} >{$}l<{$}}
\Lambda_{3,4}(1) & \frac{\pi}{6} & \Lambda_{3,4}(2) & \frac{\pi}{4} & \Lambda_{3,4}(3) & \frac{\pi}{6} &  \Lambda_{3,4}(4) & \frac{\pi}{6} & \Lambda_{3,4}(5) & \frac{\pi}{12} & \Lambda_{3,4}(6) & \frac{\pi}{4} \\ [1ex]
\Lambda_{3,4}(7) & \frac{\pi}{4} & \Lambda_{3,4}(8) & \frac{\pi}{4} & \Lambda_{3,4}(9) & 0 &  \Lambda_{3,4}(10) & \frac{\pi}{4} & \Lambda_{3,4}(11) & \frac{2\pi}{9} & \Lambda_{3,4}(12) & \frac{\pi}{4} \\ [1ex]
\Lambda_{3,4}(13) & 0 & \Lambda_{3,4}(14) & \frac{\pi}{4} & \Lambda_{3,4}(15) & \frac{7\pi}{18} &  \Lambda_{3,4}(16) & \frac{\pi}{6} & \Lambda_{3,4}(17) & \frac{\pi}{4} & \Lambda_{3,4}(18) & \frac{\pi}{12} \\ [1ex]
\Lambda_{3,4}(19) & \frac{\pi}{6} & \Lambda_{3,4}(20) & \frac{\pi}{6} & \Lambda_{3,4}(21) & \frac{\pi}{12} &  \Lambda_{3,4}(22) & \frac{\pi}{12} & \Lambda_{3,4}(23) & \frac{\pi}{12} & \Lambda_{3,4}(24) & \frac{\pi}{12}\\
\end{tabular}
}

\begin{flalign*}
\qquad & \sum_{u=1}^{24} \Lambda_{5,6}(u) = \frac{37\pi}{9} &\\
& \bm{\alpha(\gamma_5,\gamma_6)} = 0.9281 \text{, using \eqref{alphaEqn}}&
\end{flalign*}

Figure~\ref{fig:Comp2} shows the relation between the corresponding elements of $\Theta_5$ and $\Theta_6$. It also indicates, for each pair $\langle \theta_5(u), \theta_6(u) \rangle$ the corresponding point, $\langle \theta_5(u), \theta_5(u) \rangle$, on the $y=x$ line. Further, Table~\ref{tab:Comp2} provides the legend for this figure.

\begin{figure}[h]
\centering
\includegraphics[width=0.6\textwidth]{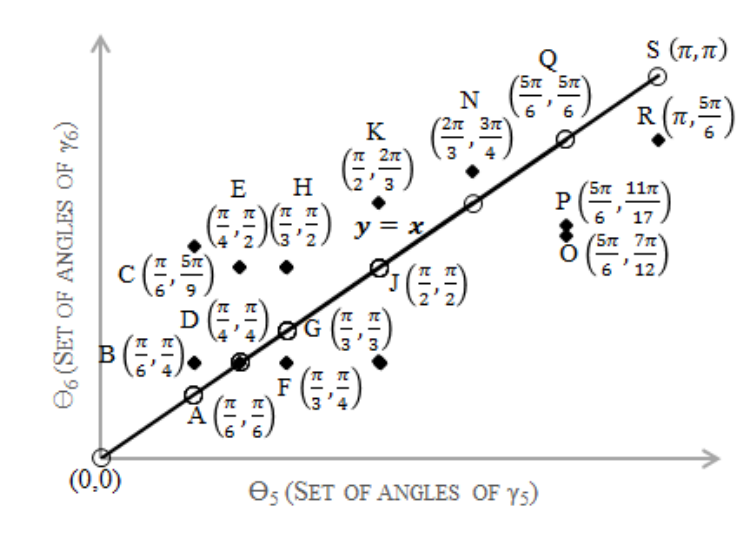}
\caption{\scriptsize Solid dots show the existing relation between $\Theta_5$ and $\Theta_6$, whereas the line passing through the hollow dots shows the expected relation between $\Theta_5$ and $\Theta_6$}
\label{fig:Comp2}
\end{figure}
\FloatBarrier 
\begin{table}
\resizebox{\textwidth}{!}{%
\begin{tabulary}{15cm}{ >{$}L<{$} @{\;$\in$\;} >{$\{}L<{\}$} >{$}L<{$} @{\;$\in$\;} >{$\{}L<{\}$}}
A & \langle \theta_5(5), \theta_5(5) \rangle, \langle \theta_5(15), \theta_5(15) \rangle, \newline \langle \theta_5(18), \theta_5(18) \rangle, \langle \theta_5(24), \theta_5(24) \rangle & B & \langle \theta_5(5), \theta_6(5) \rangle, \langle \theta_5(18), \theta_6(18) \rangle, \langle \theta_5(24), \theta_6(24) \rangle \\
C & \langle \theta_5(15), \theta_6(15) \rangle & D & \langle \theta_5(8), \theta_5(8) \rangle, \langle \theta_5(9), \theta_6(9) \rangle, \langle \theta_5(9), \theta_5(9) \rangle, \newline \langle \theta_5(12), \theta_5(12) \rangle, \langle  \theta_5(13), \theta_6(13)\rangle, \langle \theta_5(13), \theta_5(13) \rangle \\
E & \langle \theta_5(8), \theta_6(8) \rangle, \langle \theta_5(12), \theta_6(12) \rangle & F & \langle \theta_5(21), \theta_6(21) \rangle, \langle \theta_5(22), \theta_6(22) \rangle \\
G & \langle \theta_5(1), \theta_5(1) \rangle, \langle \theta_5(4), \theta_5(4) \rangle, \newline \langle \theta_5(16), \theta_5(16) \rangle, \langle \theta_5(20), \theta_5(20) \rangle, \newline \langle \theta_5(21), \theta_5(21) \rangle, \langle \theta_5(22), \theta_5(22) \rangle & H & \langle \theta_5(1), \theta_6(1) \rangle, \langle \theta_5(4), \theta_6(4) \rangle, \langle \theta_5(16), \theta_6(16) \rangle, \newline \langle \theta_5(20), \theta_6(20) \rangle \\
I & \langle \theta_5(2), \theta_6(2) \rangle, \langle \theta_5(6), \theta_6(6) \rangle, \newline \langle \theta_5(10), \theta_6(10) \rangle, \langle \theta_5(14), \theta_6(14) \rangle, \newline \langle \theta_5(17), \theta_6(17) \rangle & J & \langle \theta_5(2), \theta_5(2) \rangle, \langle \theta_5(3), \theta_5(3) \rangle, \langle \theta_5(6), \theta_5(6) \rangle, \newline \langle \theta_5(10), \theta_5(10) \rangle, \langle \theta_5(14), \theta_5(14) \rangle, \langle \theta_5(17), \theta_5(17) \rangle \\
K & \langle \theta_5(3), \theta_6(3) \rangle & M & \langle \theta_5(23), \theta_5(23) \rangle \\
N & \langle \theta_5(23), \theta_6(23) \rangle & O & \langle \theta_5(7), \theta_6(7) \rangle \\
P & \langle \theta_5(11), \theta_6(11) \rangle & Q & \langle \theta_5(7), \theta_5(7) \rangle, \langle \theta_5(11), \theta_5(11) \rangle \\
R & \langle \theta_5(19), \theta_6(19) \rangle & S & \langle \theta_5(19), \theta_5(19) \rangle \\
\end{tabulary}
}
\caption{Legend of Figure~\ref{fig:Comp2}}   
\label{tab:Comp2}
\end{table}
\FloatBarrier

Computing $\rho(\gamma_5,\gamma_6)$
\begin{flalign*}
\qquad &L_5 = \{20, 10, 10\sqrt{3}, 20, 10\sqrt{3}, 10, 7\sqrt{2}, 7, 7, 40\sqrt{2}, 40, 40, 5\sqrt{3}, 15, 10\sqrt{3}, 7, 7, 7 \} \\
&L_4 = \{ 5, 5, 5\sqrt{2}, 5(\sqrt{6} + \sqrt{2}), 5\sqrt{2}, 5(\sqrt{6} + \sqrt{2}), 5(3 + \sqrt{3}), 5\sqrt{6}(\sqrt{3} + 1), 5(3 + \sqrt{3}), \\
& \qquad \qquad 20.8075, 29.426, 20.8075, 8.217, 11.6205, 8.217, 5(3 + \sqrt{3}), 5\sqrt{6}(\sqrt{3} + 1), 5(3 + \sqrt{3}) \} 
\end{flalign*}

Table~\ref{tab:Fig5And6} indicates the input and output of IPFP
transformation.

\begin{table}[h]
\resizebox{\textwidth}{!}{%
\begin{subtable}[h]{0.6\linewidth}
\centering
\begin{tabular}{llll}
\toprule
$h$ & $\bm{l_5(h)}$ & $\bm{l_6(h)}$ & {\sc total}\\ 
\midrule
1 & 20 & 5 & 25\\  
\midrule
2 & 10 & 5 & 15\\
\midrule
3 & 17.3205 & 7.0711 & 24.3916\\
\midrule
4 & 20 & 19.3185 & 39.3185\\
\midrule
5 & 17.3205 & 27.3205 & 44.641\\
\midrule
6 & 10 & 19.3185 & 29.3185\\
\midrule
7 & 9.8995 & 23.6603 & 33.5598\\
\midrule
8 & 7 & 33.4607 & 40.4607 \\
\midrule
9 & 7 & 23.6603 & 30.6603 \\
\midrule
10 & 56.5685 & 20.8075 & 77.376 \\
\midrule
11 & 40 & 29.426 & 69.426 \\
\midrule
12 & 40 & 20.8075 & 60.8075 \\
\midrule
13 & 8.6603 & 8.217 & 16.8773 \\
\midrule
14 & 15 & 11.6205 & 26.6205 \\
\midrule
15 & 17.3205 & 8.217 & 25.5375 \\
\midrule
16 & 7 & 23.6603 & 30.6603 \\
\midrule
17 & 7 & 33.4607 & 40.4607 \\
\midrule
18 & 7 & 23.6603 & 30.6603 \\
\midrule
{\sc total} & \bf{317.0898} & \bf{343.6864} & 660.7762 \\ 
\bottomrule
\end{tabular}
\caption{Values on which IPFP is to be performed}
\label{tab:Comp2Table1}
\end{subtable}%
\begin{subtable}[h]{0.6\linewidth}
\centering
\begin{tabular}{llll}
\toprule
$h$ & $\bm{l_5'(h)}$ & $\bm{l_6'(h)}$ & {\sc total}\\ 
\midrule
1 & 11.9969 & 13.0031 & 25\\  
\midrule
2 & 7.1981 & 7.8019 & 15\\
\midrule
3 & 11.7049 & 12.6867 & 24.3916\\
\midrule
4 & 18.868 & 20.4505 & 39.3185\\
\midrule
5 & 21.4221 & 23.2189 & 44.641\\
\midrule
6 & 14.0692 & 15.2493 & 29.3185\\
\midrule
7 & 16.1045 & 17.4553 & 33.5598\\
\midrule
8 & 19.4161 & 21.0446 & 40.4607 \\
\midrule
9 & 14.7131 & 15.9472 & 30.6603 \\
\midrule
10 & 37.1308 & 40.2452 & 77.376 \\
\midrule
11 & 33.3158 & 36.1102 & 69.426 \\
\midrule
12 & 29.18 & 31.6275 & 60.8075 \\
\midrule
13 & 8.099 & 8.7783 & 16.8773 \\
\midrule
14 & 12.7745 & 13.846 & 26.6205 \\
\midrule
15 & 12.2548 & 13.2827 & 25.5375 \\
\midrule
16 & 14.7131 & 15.9472 & 30.6603 \\
\midrule
17 & 19.4161 & 21.0446 & 40.4607 \\
\midrule
18 & 14.7131 & 15.9472 & 30.6603 \\
\midrule
{\sc total} & \bf{317.0898} & \bf{343.6864} & 660.7762 \\ 
\bottomrule
\end{tabular}
\caption{Values obtained on applying IPFP}
\label{tab:Comp2Table2}
\end{subtable}
}
\caption{IPFP Transformation}
\label{tab:Fig5And6}
\end{table}
	
\begin{flalign*}
\qquad & \text{Considering any row from the Table~\ref{tab:Comp2Table2}, we compute $m$} & \\
&  m = \frac{13.0031}{11.9969} = 1.0839 &\\
& y =  1.0839 x \text{, equation of the expected line} &
\end{flalign*}

Figure~\ref{fig:Comp2-edge} shows the relation between the
corresponding elements of $L_5$ and $L_6$. It also indicates, for each
pair $\langle l_5(h), l_6(h) \rangle$ the corresponding point,
$\langle l_5'(h), l_6'(h) \rangle$, on the $y=1.0839 x$ line. Further,
Table~\ref{tab:Comp2-edge} provides the legend for this figure.

\begin{figure}[h]
\centering
\includegraphics[width=0.6\textwidth]{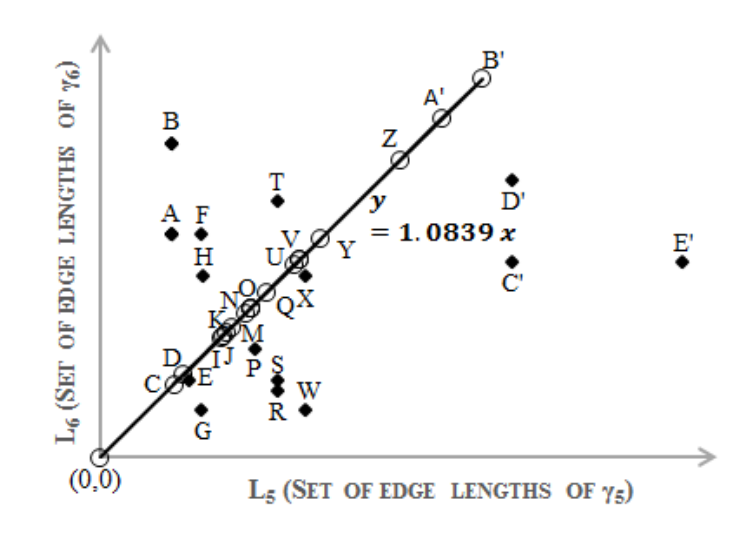}
\caption{\scriptsize Solid dots show the existing relation between $L_5$ and $L_6$, whereas the line passing through the hollow dots shows the expected relation between $L_5$ and $L_6$}
\label{fig:Comp2-edge}
\end{figure}
\FloatBarrier 
\begin{longtable}{ >{$}l<{$} >{$}l<{$} }
{\!\begin{aligned} 
               A &= \langle 7, 23.6603 \rangle  \\   
                  &\in \{ \langle l_5(9), l_6(9) \rangle, \langle l_5(16), l_6(16) \rangle, \\
                  & \quad \langle l_5(18), l_6(18) \rangle \} \end{aligned}} & {\!\begin{aligned} 
               B &= \langle 7, 33.4607 \rangle  \\   
                  &\in \{ \langle l_5(8), l_6(8) \rangle, \langle l_5(17), l_6(17) \rangle \} \end{aligned}} \\
{\!\begin{aligned} 
               C &= \langle 7.1981, 7.8019 \rangle  \\   
                  &\in \{ \langle l_5'(2), l_6'(2) \rangle \} \end{aligned}} & {\!\begin{aligned} 
               D &= \langle 8.099, 8.7783 \rangle  \\   
                  &\in \{ \langle l_5'(13), l_6'(13) \rangle \} \end{aligned}} \\
{\!\begin{aligned} 
               E &= \langle 8.6603, 8.217 \rangle  \\   
                  &\in \{ \langle l_5(13), l_6(13) \rangle \} \end{aligned}} & {\!\begin{aligned} 
               F &= \langle 9.8995, 23.6603 \rangle  \\   
                  &\in \{ \langle l_5(7), l_6(7) \rangle \} \end{aligned}} \\
{\!\begin{aligned} 
               G &= \langle 10, 5 \rangle  \\   
                  &\in \{ \langle l_5(2), l_6(2) \rangle \} \end{aligned}} & {\!\begin{aligned} 
               H &= \langle 10, 19.3185 \rangle  \\   
                  &\in \{ \langle l_5(6), l_6(6) \rangle \} \end{aligned}} \\
{\!\begin{aligned} 
               I &= \langle 11.7049, 12.6867 \rangle  \\   
                  &\in \{ \langle l_5'(3), l_6'(3) \rangle \} \end{aligned}} & {\!\begin{aligned} 
               J &= \langle 11.9969, 13.0031 \rangle  \\   
                  &\in \{ \langle l_5'(1), l_6'(1) \rangle \} \end{aligned}} \\
{\!\begin{aligned} 
               K &= \langle 12.2548, 13.2827 \rangle  \\   
                  &\in \{ \langle l_5'(15), l_6'(15) \rangle \} \end{aligned}} & {\!\begin{aligned} 
               M &= \langle 12.7745, 13.846 \rangle  \\   
                  &\in \{ \langle l_5'(14), l_6'(14) \rangle \} \end{aligned}} \\   
{\!\begin{aligned} 
               N &= \langle 14.0692, 15.2493 \rangle  \\   
                  &\in \{ \langle l_5'(6), l_6'(6) \rangle \} \end{aligned}} & {\!\begin{aligned} 
               O &= \langle 14.7131, 15.9472 \rangle  \\   
                  &\in \{ \langle l_5'(9), l_6'(9) \rangle, \langle l_5'(16), l_6'(16) \rangle, \\
                  & \quad \langle l_5'(18), l_6'(18) \rangle \} \end{aligned}} \\   
{\!\begin{aligned} 
               P &= \langle 15, 11.6205 \rangle  \\   
                  &\in \{ \langle l_5(14), l_6(14) \rangle \} \end{aligned}} & {\!\begin{aligned} 
               Q &= \langle 16.1045, 17.4553 \rangle  \\   
                  &\in \{ \langle l_5'(7), l_6'(7) \rangle \} \end{aligned}} \\      
{\!\begin{aligned} 
               R &= \langle 17.3205, 7.0711 \rangle  \\   
                  &\in \{ \langle l_5(3), l_6(3) \rangle \} \end{aligned}} & {\!\begin{aligned} 
               S &= \langle 17.3205, 8.217 \rangle  \\   
                  &\in \{ \langle l_5(15), l_6(15) \rangle \} \end{aligned}} \\                  
{\!\begin{aligned} 
               T &= \langle 17.3205, 27.3205 \rangle  \\   
                  &\in \{ \langle l_5(5), l_6(5) \rangle \} \end{aligned}} & {\!\begin{aligned} 
               U &= \langle 18.868, 20.4505 \rangle  \\   
                  &\in \{ \langle l_5'(4), l_6'(4) \rangle \} \end{aligned}} \\                  
{\!\begin{aligned} 
               V &= \langle 19.4161, 21.0446 \rangle  \\   
                  &\in \{ \langle l_5'(8), l_6'(8) \rangle, \langle l_5'(17), l_6'(17) \rangle \} \end{aligned}} & {\!\begin{aligned} 
               W &= \langle 20, 5 \rangle  \\   
                  &\in \{ \langle l_5(1), l_6(1) \rangle \} \end{aligned}} \\                  
{\!\begin{aligned} 
               X &= \langle 20, 19.3185 \rangle  \\   
                  &\in \{ \langle l_5(4), l_6(4) \rangle \} \end{aligned}} & {\!\begin{aligned} 
               Y &= \langle 21.4221, 23.2189 \rangle  \\   
                  &\in \{ \langle l_5'(5), l_6'(5) \rangle \} \end{aligned}} \\                 
{\!\begin{aligned} 
               Z &= \langle 29.18, 31.6275 \rangle  \\   
                  &\in \{ \langle l_5'(12), l_6'(12) \rangle \} \end{aligned}} & {\!\begin{aligned} 
               A' &= \langle 33.3158, 36.1102 \rangle  \\   
                  &\in \{ \langle l_5'(11), l_6'(11) \rangle \} \end{aligned}} \\  
{\!\begin{aligned} 
               B' &= \langle 37.1308, 40.2452 \rangle  \\   
                  &\in \{ \langle l_5'(10), l_6'(10) \rangle \} \end{aligned}} & {\!\begin{aligned} 
               C' &= \langle 40, 20.8075 \rangle  \\   
                  &\in \{ \langle l_5(12), l_6(12) \rangle \} \end{aligned}} \\  
{\!\begin{aligned} 
               D' &= \langle 40, 29.426 \rangle  \\   
                  &\in \{ \langle l_5(11), l_6(11) \rangle \} \end{aligned}} & {\!\begin{aligned} 
               E' &= \langle 56.5685, 20.8075 \rangle  \\   
                  &\in \{ \langle l_5(10), l_6(10) \rangle \} \end{aligned}} \\                       
\caption{Legend of Figure~\ref{fig:Comp2-edge}}   
\label{tab:Comp2-edge}
\end{longtable}
\FloatBarrier 

Now, the Euclidean Distance is computed using \eqref{euclideanEqn} 

\begin{tabular}{ >{$}l<{$} @{\;=\;} >{$}l<{$} >{$}l<{$} @{\;=\;} >{$}l<{$} >{$}l<{$} @{\;=\;} >{$}l<{$} >{$}l<{$} @{\;=\;} >{$}l<{$}}
\Delta_{5,6}(1) & 11.3181 & \Delta_{5,6}(2) & 3.9625 & \Delta_{5,6}(3) & 7.9417 & \Delta_{5,6}(4) & 1.6009 \\ 
\Delta_{5,6}(5) & 5.8005 & \Delta_{5,6}(6) & 5.7547 & \Delta_{5,6}(7) & 8.7752 & \Delta_{5,6}(8) & 17.5589  \\  
\Delta_{5,6}(9) & 10.9079 & \Delta_{5,6}(10) & 27.4891 & \Delta_{5,6}(11) & 9.4529 & \Delta_{5,6}(12) & 15.3018 \\
\Delta_{5,6}(13) & 0.7938 & \Delta_{5,6}(14) & 3.1473 & \Delta_{5,6}(15) & 7.164 & \Delta_{5,6}(16) & 10.9079\\
\Delta_{5,6}(17) & 17.5589 & \Delta_{5,6}(18) & 10.9079  \\
\end{tabular}
\begin{flalign*}
\qquad & \sum_{h=1}^{18} \Delta_{5,6}(h) = 176.3443 & \\
& \bm{\rho(\gamma_5,\gamma_6)} = 0.9074 \text{, using \eqref{rhoEqn}}&
\end{flalign*}

Computing $d(\gamma_5,\gamma_6)$
\begin{flalign*}
\qquad & \bm{d(\gamma_5,\gamma_6)} = 0.9177 \text{, with } \beta = 0.5 \text{, using \eqref{dDefinitionEqn}}&
\end{flalign*}
\FloatBarrier

\section{IPFP: step-by-step}
\label{appendix:Example4}
Consider Figures~\ref{fig:octAndQuad1} and \ref{fig:octAndQuad2}

\begin{figure}[h]
\centering
\begin{subfigure}[b]{0.45\textwidth}
\includegraphics[width=7.5cm]{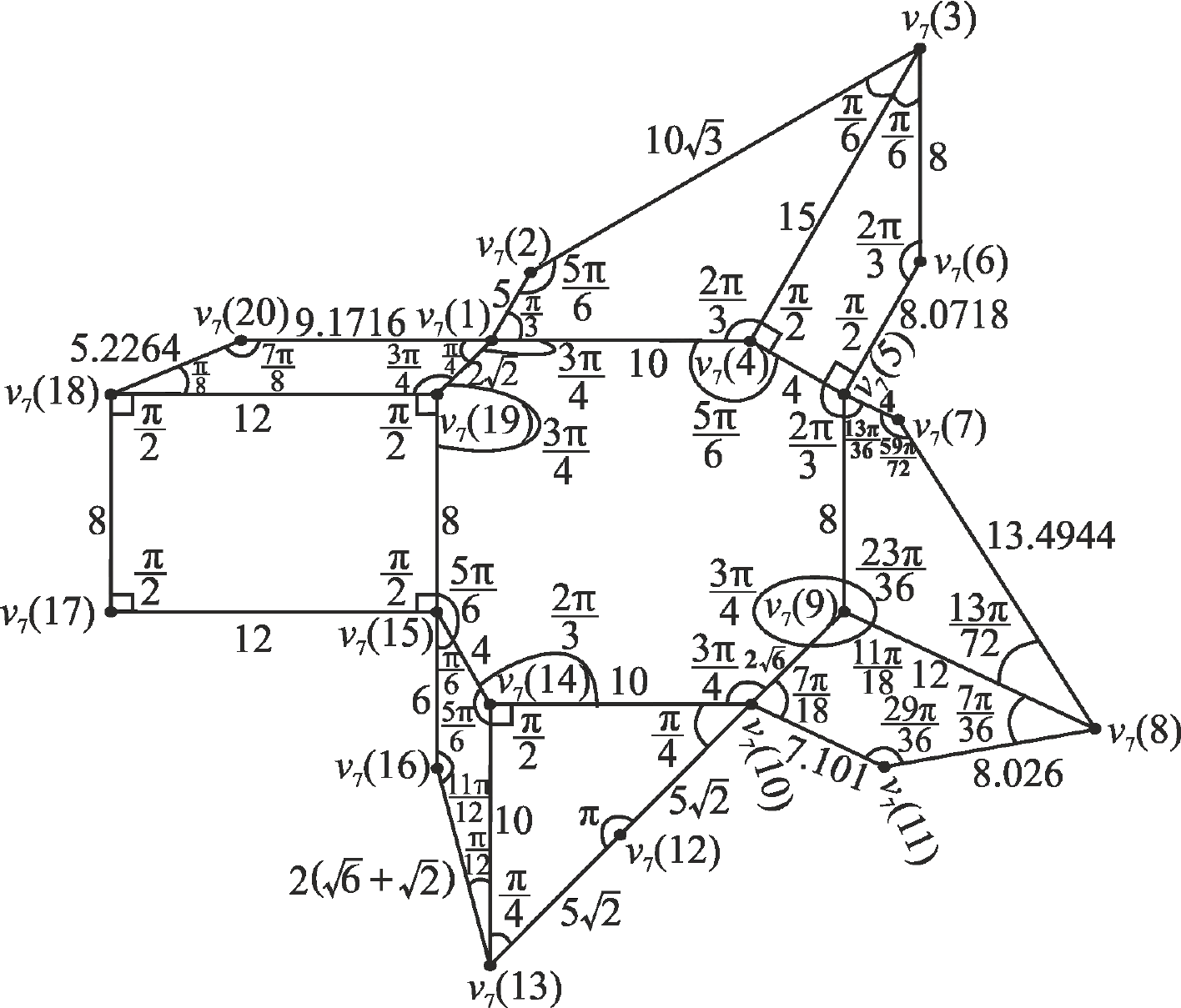}
\caption{$\bm{\gamma_7}$}
\label{fig:octAndQuad1}
\end{subfigure}
\hfill
\begin{subfigure}[b]{0.45\textwidth}
\includegraphics[width=8cm]{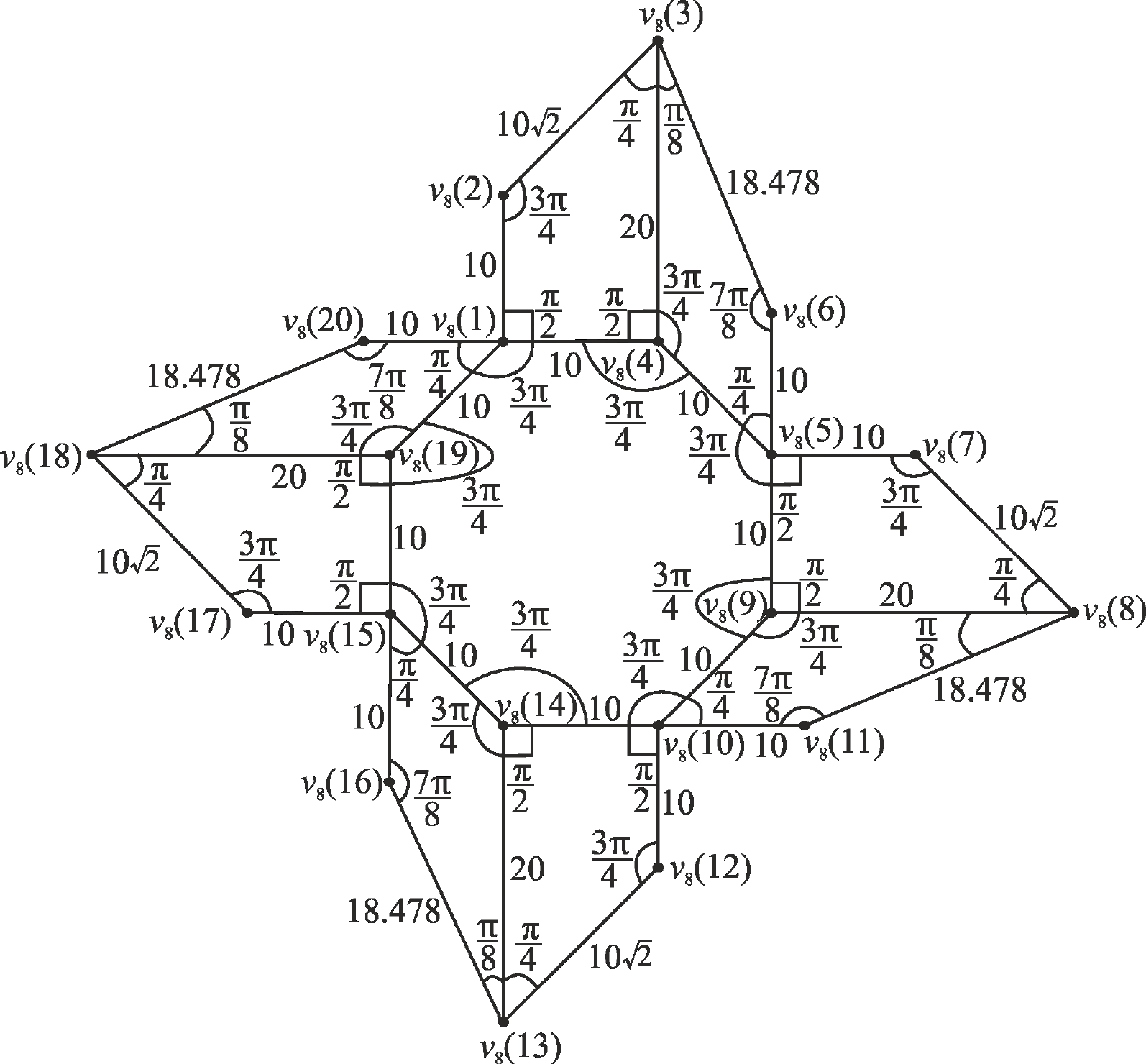}
\caption{$\bm{\gamma_8$}}
\label{fig:octAndQuad2}
\end{subfigure}
\caption{}
\end{figure}

Computing $\alpha(\gamma_7,\gamma_8)$
\begin{flalign*}
\qquad & \Theta_7 = \left\{ \frac{\pi}{3}, \frac{5\pi}{6}, \frac{\pi}{6}, \frac{2\pi}{3}, \frac{5\pi}{6}, \frac{\pi}{2}, \frac{\pi}{6}, \frac{2\pi}{3}, \frac{\pi}{2}, \frac{2\pi}{3}, \frac{13\pi}{36}, \frac{59\pi}{72}, \frac{13\pi}{72}, \frac{23\pi}{36}, \frac{3\pi}{4}, \frac{11\pi}{18}, \frac{7\pi}{36}, \frac{29\pi}{36}, \right. \\
& \left. \frac{7\pi}{18}, \frac{3\pi}{4}, \frac{\pi}{4}, \pi, \frac{\pi}{4}, \frac{\pi}{2}, \frac{2\pi}{3}, \frac{5\pi}{6}, \frac{\pi}{12}, \frac{11\pi}{12}, \frac{\pi}{6}, \frac{5\pi}{6}, \frac{\pi}{2}, \frac{\pi}{2}, \frac{\pi}{2}, \frac{\pi}{2}, \frac{3\pi}{4}, \frac{3\pi}{4}, \frac{\pi}{8}, \frac{7\pi}{8}, \frac{\pi}{4}, \frac{3\pi}{4} \right\} & \\
& \Theta_8= \left\{ \frac{\pi}{2}, \frac{3\pi}{4}, \frac{\pi}{4}, \frac{\pi}{2}, \frac{3\pi}{4}, \frac{3\pi}{4}, \frac{\pi}{8}, \frac{7\pi}{8}, \frac{\pi}{4}, \frac{3\pi}{4}, \frac{\pi}{2}, \frac{3\pi}{4}, \frac{\pi}{4}, \frac{\pi}{2}, \frac{3\pi}{4}, \frac{3\pi}{4}, \frac{\pi}{8}, \frac{7\pi}{8}, \frac{\pi}{4}, \frac{3\pi}{4}, \frac{\pi}{2}, \right. \\
& \left.  \frac{3\pi}{4}, \frac{\pi}{4}, \frac{\pi}{2}, \frac{3\pi}{4}, \frac{3\pi}{4}, \frac{\pi}{8}, \frac{7\pi}{8}, \frac{\pi}{4}, \frac{3\pi}{4}, \frac{\pi}{2}, \frac{3\pi}{4}, \frac{\pi}{4}, \frac{\pi}{2}, \frac{3\pi}{4}, \frac{3\pi}{4}, \frac{\pi}{8}, \frac{7\pi}{8}, \frac{\pi}{4}, \frac{3\pi}{4} \right\} & \\
\end{flalign*}

Using \eqref{angularShiftEqn} we compute the Euclidean distance.

\begin{tabular}{ >{$}l<{$} @{\;=\;} >{$}l<{$} >{$}l<{$} @{\;=\;} >{$}l<{$} >{$}l<{$} @{\;=\;} >{$}l<{$} >{$}l<{$} @{\;=\;} >{$}l<{$} >{$}l<{$} @{\;=\;} >{$}l<{$}}
\Lambda_{7,8}(1) & \frac{\pi}{6} & \Lambda_{7,8}(2) & \frac{\pi}{12} & \Lambda_{7,8}(3) & \frac{\pi}{12} &  \Lambda_{7,8}(4) & \frac{\pi}{6} & \Lambda_{7,8}(5) & \frac{\pi}{12} \\ [1ex]
\Lambda_{7,8}(6) & \frac{\pi}{4} & \Lambda_{7,8}(7) & \frac{\pi}{24} & \Lambda_{7,8}(8) & \frac{5\pi}{24}  & \Lambda_{7,8}(9) & \frac{\pi}{4} &  \Lambda_{7,8}(10) & \frac{\pi}{12} \\ [1ex]
\Lambda_{7,8}(11) & \frac{5\pi}{36} & \Lambda_{7,8}(12) & \frac{5\pi}{72} & \Lambda_{7,8}(13) & \frac{5\pi}{72} & \Lambda_{7,8}(14) & \frac{5\pi}{36} & \Lambda_{7,8}(15) & 0 \\ [1ex]
\Lambda_{7,8}(16) & \frac{5\pi}{36} & \Lambda_{7,8}(17) & \frac{5\pi}{72} & \Lambda_{7,8}(18) & \frac{5\pi}{72} & \Lambda_{7,8}(19) & \frac{5\pi}{36} & \Lambda_{7,8}(20) & 0 \\ [1ex]
\Lambda_{7,8}(21) & \frac{\pi}{4} &  \Lambda_{7,8}(22) & \frac{\pi}{4} & \Lambda_{7,8}(23) & 0 & \Lambda_{7,8}(24) & 0 & \Lambda_{7,8}(25) & \frac{\pi}{12} \\ [1ex]
\Lambda_{7,8}(26) & \frac{\pi}{12} & \Lambda_{7,8}(27) & \frac{\pi}{24} & \Lambda_{7,8}(28) & \frac{\pi}{24} & \Lambda_{7,8}(29) & \frac{\pi}{12} & \Lambda_{7,8}(30) & \frac{\pi}{12} \\ [1ex]
\Lambda_{7,8}(31) & 0 & \Lambda_{7,8}(32) & \frac{\pi}{4} & \Lambda_{7,8}(33) & \frac{\pi}{4} & \Lambda_{7,8}(34) & 0 & \Lambda_{7,8}(35) & 0 \\ [1ex]
\Lambda_{7,8}(36) & 0 & \Lambda_{7,8}(37) & 0 & \Lambda_{7,8}(38) & 0 & \Lambda_{7,8}(39) & 0 & \Lambda_{7,8}(40) & 0 \\ 
\end{tabular}
\begin{flalign*}
\qquad & \sum_{u=1}^{40} \Lambda_{7,8}(u) = \frac{11\pi}{3} &\\
& \bm{\alpha(\gamma_7,\gamma_8)} = 0.9201 \text{, using \eqref{alphaEqn}}&
\end{flalign*}

Figure~\ref{fig:Comp3} shows the relation between the corresponding elements of $\Theta_7$ and $\Theta_8$. It also indicates, for each pair $\langle \theta_7(u), \theta_8(u) \rangle$ the corresponding point, $\langle \theta_7(u), \theta_7(u) \rangle$, on the $y=x$ line. Further, Table~\ref{tab:Comp3} provides the legend for this figure.
\begin{figure}[h]
\centering
\includegraphics[width=0.6\textwidth]{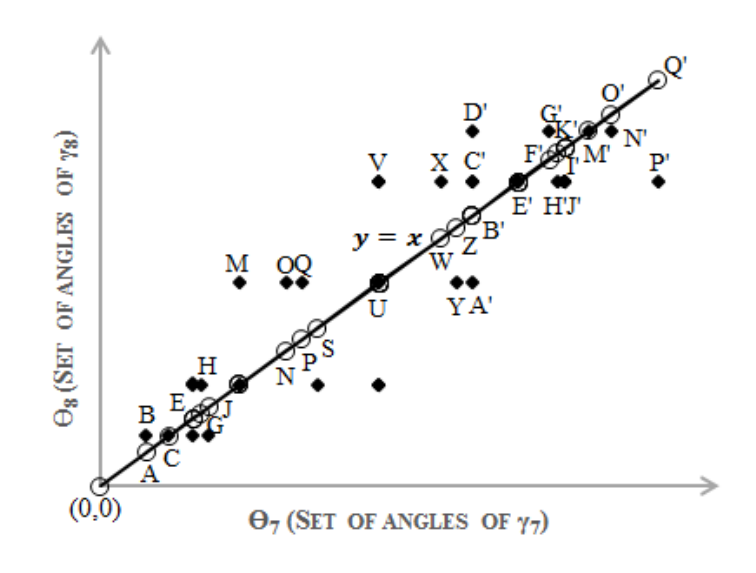}
\caption{\scriptsize Solid dots show the existing relation between $\Theta_7$ and $\Theta_8$, whereas the line passing through the hollow dots shows the expected relation between $\Theta_7$ and $\Theta_8$}
\label{fig:Comp3}
\end{figure}
\FloatBarrier
\begin{longtable}{ >{$}l<{$} >{$}l<{$} }
{\!\begin{aligned} 
               A &= \left \langle \frac{\pi}{12}, \frac{\pi}{12} \right \rangle  \\   
                  &\in \{ \langle \theta_7(27), \theta_7(27) \rangle \} \end{aligned}} & {\!\begin{aligned} 
               B &= \left \langle \frac{\pi}{12}, \frac{\pi}{8} \right \rangle  \\   
                  &\in \{ \langle \theta_7(27), \theta_8(27) \rangle \} \end{aligned}} \\
{\!\begin{aligned} 
               C &= \left \langle \frac{\pi}{8}, \frac{\pi}{8} \right \rangle  \\   
                  &\in \{ \langle \theta_7(37), \theta_8(37) \rangle, \langle \theta_7(37), \theta_7(37) \rangle \} \end{aligned}} & {\!\begin{aligned} 
               D &= \left \langle \frac{\pi}{6}, \frac{\pi}{8} \right \rangle  \\   
                  &\in \{ \langle \theta_7(7), \theta_8(7) \rangle \} \end{aligned}} \\
{\!\begin{aligned} 
               E &= \left \langle \frac{\pi}{6}, \frac{\pi}{6} \right \rangle  \\   
                  &\in \{ \langle \theta_7(3), \theta_7(3) \rangle, \langle (\theta_7(7), \theta_7(7) \rangle, \\
                  & \quad \langle \theta_7(29), \theta_7(29) \rangle \} \end{aligned}} & {\!\begin{aligned} 
               F &= \left \langle \frac{\pi}{6}, \frac{\pi}{4} \right \rangle  \\   
                  &\in \{ \langle \theta_7(3), \theta_8(3) \rangle, \langle \theta_7(29), \theta_8(29) \rangle \} \end{aligned}} \\
{\!\begin{aligned} 
               G &= \left \langle \frac{13\pi}{72}, \frac{13\pi}{72} \right \rangle  \\   
                  &\in \{ \langle \theta_7(13), \theta_7(13) \rangle \} \end{aligned}} & {\!\begin{aligned} 
               H &= \left \langle \frac{13\pi}{72}, \frac{\pi}{4} \right \rangle  \\   
                  &\in \{ \langle \theta_7(13), \theta_8(13) \rangle \} \end{aligned}} \\
{\!\begin{aligned} 
               I &= \left \langle \frac{7\pi}{36}, \frac{\pi}{8} \right \rangle  \\   
                  &\in \{ \langle \theta_7(17), \theta_8(17) \rangle \} \end{aligned}} & {\!\begin{aligned} 
               J &= \left \langle \frac{7\pi}{36}, \frac{7\pi}{36} \right \rangle  \\   
                  &\in \{ \langle \theta_7(17), \theta_7(17) \rangle \} \end{aligned}} \\             
{\!\begin{aligned} 
               K &= \left \langle \frac{\pi}{4}, \frac{\pi}{4} \right \rangle  \\   
                  &\in \{ \langle \theta_7(21), \theta_7(21) \rangle, \langle \theta_7(23), \theta_8(23) \rangle, \\
                  & \quad \langle \theta_7(23), \theta_7(23) \rangle, \langle \theta_7(39), \theta_8(39) \rangle, \\
                  & \quad \langle \theta_7(39), \theta_7(39) \rangle \} \end{aligned}} & {\!\begin{aligned} 
               M &= \left \langle \frac{\pi}{4}, \frac{\pi}{2} \right \rangle  \\   
                  &\in \{ \langle \theta_7(21), \theta_8(21) \rangle \} \end{aligned}} \\                      
{\!\begin{aligned} 
               N &= \left \langle \frac{\pi}{3}, \frac{\pi}{3} \right \rangle  \\   
                  &\in \{ \langle \theta_7(1), \theta_7(1) \rangle \} \end{aligned}} & {\!\begin{aligned} 
               O &= \left \langle \frac{\pi}{3}, \frac{\pi}{2} \right \rangle  \\   
                  &\in \{ \langle \theta_7(1), \theta_8(1) \rangle \} \end{aligned}} \\   
{\!\begin{aligned} 
              P &= \left \langle \frac{13\pi}{36}, \frac{13\pi}{36} \right \rangle  \\   
                  &\in \{ \langle \theta_7(11), \theta_7(11) \rangle \} \end{aligned}} & {\!\begin{aligned} 
              Q &= \left \langle \frac{13\pi}{36}, \frac{\pi}{2} \right \rangle  \\   
                  &\in \{ \langle \theta_7(11), \theta_8(11) \rangle \} \end{aligned}} \\                     
{\!\begin{aligned} 
              R &= \left \langle \frac{7\pi}{18}, \frac{\pi}{4} \right \rangle  \\   
                  &\in \{ \langle \theta_7(19), \theta_8(19) \rangle \} \end{aligned}} & {\!\begin{aligned} 
              S &= \left \langle \frac{7\pi}{18}, \frac{7\pi}{18} \right \rangle  \\   
                  &\in \{ \langle \theta_7(19), \theta_7(19) \rangle \} \end{aligned}} \\  
{\!\begin{aligned} 
              T &= \left \langle \frac{\pi}{2}, \frac{\pi}{4} \right \rangle  \\   
                  &\in \{ \langle \theta_7(9), \theta_8(9) \rangle, \langle \theta_7(33), \theta_8(33) \rangle \} \end{aligned}} & {\!\begin{aligned} 
              U &= \left \langle \frac{\pi}{2}, \frac{\pi}{2} \right \rangle  \\   
                  &\in \{ \langle \theta_7(6), \theta_7(6) \rangle, \langle \theta_7(9), \theta_7(9) \rangle, \langle \theta_7(24), \theta_8(24) \rangle, \\
                  & \quad \langle \theta_7(24), \theta_7(24) \rangle, \langle \theta_7(31), \theta_8(31) \rangle, \langle \theta_7(31), \theta_7(31 ) \rangle, \\
                  & \quad \langle \theta_7(32), \theta_7(32) \rangle, \langle \theta_7(33), \theta_7(33) \rangle, \langle \theta_7(34), \theta_8(34) \rangle, \\
                  & \quad \langle \theta_7(34), \theta_7(34) \rangle\} \end{aligned}} \\                 
{\!\begin{aligned} 
               V &= \left \langle \frac{\pi}{2}, \frac{3\pi}{4} \right \rangle  \\   
                  &\in \{ \langle \theta_7(6), \theta_8(6) \rangle, \langle \theta_7(32), \theta_8(32) \rangle \} \end{aligned}} & {\!\begin{aligned} 
               W &= \left \langle \frac{11\pi}{18}, \frac{11\pi}{18} \right \rangle  \\   
                  &\in \{ \langle \theta_7(16), \theta_7(16) \rangle \} \end{aligned}} \\           
{\!\begin{aligned} 
               X &= \left \langle \frac{11\pi}{18}, \frac{3\pi}{4} \right \rangle  \\   
                  &\in \{ \langle \theta_7(16), \theta_8(16) \rangle \} \end{aligned}} & {\!\begin{aligned} 
               Y &= \left \langle \frac{23\pi}{36}, \frac{\pi}{2} \right \rangle  \\   
                  &\in \{ \langle \theta_7(14), \theta_8(14) \rangle \} \end{aligned}} \\  
{\!\begin{aligned} 
               X &= \left \langle \frac{11\pi}{18}, \frac{3\pi}{4} \right \rangle  \\   
                  &\in \{ \langle \theta_7(16), \theta_8(16) \rangle \} \end{aligned}} & {\!\begin{aligned} 
               Y &= \left \langle \frac{23\pi}{36}, \frac{\pi}{2} \right \rangle  \\   
                  &\in \{ \langle \theta_7(14), \theta_8(14) \rangle \} \end{aligned}} \\     
{\!\begin{aligned} 
               Z &= \left \langle \frac{23\pi}{36}, \frac{23\pi}{36} \right \rangle  \\   
                  &\in \{ \langle \theta_7(14), \theta_7(14) \rangle \} \end{aligned}} & {\!\begin{aligned} 
               A' &= \left \langle \frac{2\pi}{3}, \frac{\pi}{2} \right \rangle  \\   
                  &\in \{ \langle \theta_7(4), \theta_8(4) \rangle \} \end{aligned}} \\  
{\!\begin{aligned} 
               B' &= \left \langle \frac{2\pi}{3}, \frac{2\pi}{3} \right \rangle  \\   
                  &\in \{ \langle \theta_7(4), \theta_7(4) \rangle, \langle \theta_7(8), \theta_7(8) \rangle, \\
                  & \quad \langle \theta_7(10), \theta_7(10) \rangle, \langle \theta_7(25), \theta_7(25) \rangle \} \end{aligned}} & {\!\begin{aligned} 
               C' &= \left \langle \frac{2\pi}{3}, \frac{3\pi}{4} \right \rangle  \\   
                  &\in \{ \langle \theta_7(10), \theta_8(10) \rangle, \langle \theta_7(25), \theta_8(25) \rangle \} \end{aligned}} \\                 
{\!\begin{aligned} 
               D' &= \left \langle \frac{2\pi}{3}, \frac{7\pi}{8} \right \rangle  \\   
                  &\in \{ \langle \theta_7(8), \theta_8(8) \rangle \} \end{aligned}} & {\!\begin{aligned} 
               E' &= \left \langle \frac{3\pi}{4}, \frac{3\pi}{4} \right \rangle  \\   
                  &\in \{ \langle \theta_7(15), \theta_8(15) \rangle, \langle \theta_7(15), \theta_7(15) \rangle, \langle \theta_7(20), \theta_8(20) \rangle, \\
                  & \quad \langle \theta_7(20), \theta_7(20) \rangle, \langle \theta_7(35), \theta_8(35) \rangle, \langle \theta_7(35), \theta_7(35) \rangle, \\
                  & \quad \langle \theta_7(36), \theta_8(36) \rangle, \langle \theta_7(36), \theta_7(36) \rangle, \langle \theta_7(40), \theta_8(40) \rangle, \\
                  & \quad \langle \theta_7(40), \theta_7(40) \rangle \} \end{aligned}} \\           
{\!\begin{aligned} 
               F' &= \left \langle \frac{29\pi}{36}, \frac{29\pi}{36} \right \rangle  \\   
                  &\in \{ \langle \theta_7(18), \theta_7(18) \rangle \} \end{aligned}} & {\!\begin{aligned} 
               G' &= \left \langle \frac{29\pi}{36}, \frac{7\pi}{8} \right \rangle  \\   
                  &\in \{ \langle \theta_7(18), \theta_8(18) \rangle \} \end{aligned}} \\  
{\!\begin{aligned} 
               H' &= \left \langle \frac{59\pi}{72}, \frac{3\pi}{4} \right \rangle  \\   
                  &\in \{ \langle \theta_7(12), \theta_8(12) \rangle \} \end{aligned}} & {\!\begin{aligned} 
               I' &= \left \langle \frac{59\pi}{72}, \frac{59\pi}{72} \right \rangle  \\   
                  &\in \{ \langle \theta_7(12), \theta_7(12) \rangle \} \end{aligned}} \\   
{\!\begin{aligned} 
               J' &= \left \langle \frac{5\pi}{6}, \frac{3\pi}{4} \right \rangle  \\   
                  &\in \{ \langle \theta_7(2), \theta_8(2) \rangle, \langle \theta_7(5), \theta_8(5) \rangle, \\
                  & \quad \langle \theta_7(26), \theta_8(26) \rangle, \langle \theta_7(30), \theta_8(30) \rangle \} \end{aligned}} & {\!\begin{aligned} 
               K' &= \left \langle \frac{5\pi}{6}, \frac{5\pi}{6} \right \rangle  \\   
                  &\in \{ \langle \theta_7(2), \theta_7(2) \rangle, \langle \theta_7(5), \theta_7(5) \rangle, \langle \theta_7(26), \theta_7(26) \rangle, \\
                  & \quad \langle \theta_7(30), \theta_7(30) \rangle \} \end{aligned}} \\          
{\!\begin{aligned} 
               M' &= \left \langle \frac{7\pi}{8}, \frac{7\pi}{8} \right \rangle  \\   
                  &\in \{ \langle \theta_7(38), \theta_8(38) \rangle, \langle \theta_7(38), \theta_7(38) \rangle \} \end{aligned}} & {\!\begin{aligned} 
               N' &= \left \langle \frac{11\pi}{12}, \frac{7\pi}{8} \right \rangle  \\   
                  &\in \{ \langle \theta_7(28), \theta_8(28) \rangle \} \end{aligned}} \\                                       
{\!\begin{aligned} 
               O' &= \left \langle \frac{11\pi}{12}, \frac{11\pi}{12} \right \rangle  \\   
                  &\in \{ \langle \theta_7(28), \theta_7(28) \rangle \} \end{aligned}} & {\!\begin{aligned} 
               P' &= \left \langle \pi, \frac{3\pi}{4} \right \rangle  \\   
                  &\in \{ \langle \theta_7(22), \theta_8(22) \rangle \} \end{aligned}} \\                       
{\!\begin{aligned} 
               Q' &= \left \langle \pi, \pi  \right \rangle  \\   
                  &\in \{ \langle \theta_7(22), \theta_7(22) \rangle \} \end{aligned}} \\     
                  \\
\caption{Legend of Table~\ref{fig:Comp3}}            
\label{tab:Comp3}                                                                    
\end{longtable}
\FloatBarrier

Computing $\rho(\gamma_7,\gamma_8)$
\begin{flalign*}
\qquad &L_7 = \{10, 5, 10\sqrt{3}, 15, 4, 8.0718, 8, 8, 4, 13.4944, 12, 2\sqrt{6}, 7.101, 8.026, 10, 5\sqrt{2}, 5\sqrt{2}, 10, \\
& 4, 6, 2(\sqrt{6}+\sqrt{2}), 8, 12, 8, 12, 2\sqrt{2}, 9.1716, 5.2264 \} \\
&L_8 = \{ 10, 10, 10\sqrt{2}, 20, 10, 10, 18.478, 10, 10, 10\sqrt{2}, 20, 10, 10, 18.478, 10, 10, 10\sqrt{2}, \\
& 20, 10, 10, 18.478, 10, 10, 10\sqrt{2}, 20, 10, 10, 18.478 \} 
\end{flalign*}

Tables~\ref{tab:IPFP-1}, ~\ref{tab:IPFP-2}, ~\ref{tab:IPFP-3} and ~\ref{tab:IPFP-4} give a detailed explanation of the IPFP transformation used in this paper. 
\begin{enumerate}
\item Table~\ref{tab:IPFP-1} gives the row sum and column sum that will be maintained in row fitting and column fitting respectively.
\item Table~\ref{tab:IPFP-2}, is the initial table, on which row fitting of the first iteration is performed.
\item Table~\ref{tab:IPFP-3}, is the result of row fitting. The value of each cell is obtained as follows:
$$r_{n,o} = \frac{q_{n,o} * s_r(n)}{s_q(n)}$$
where, $r_{n,o}$ represents the value in nth row and oth column of Table~\ref{tab:IPFP-3}, $q_{n,o}$ represents the value in nth row and oth column of Table~\ref{tab:IPFP-2}, $s_r(n)$ represents the nth row sum of Table~\ref{tab:IPFP-3} and $s_q(n)$ represents the nth row sum of Table~\ref{tab:IPFP-2}.
\item Table~\ref{tab:IPFP-4}, is the result of column fitting. The value of each cell is obtained as follows:
$$c_{n,o} = \frac{r_{n,o} * s_c(o)}{s_r(o)}$$
where, $c_{n,o}$ represents the value in nth row and oth column of Table~\ref{tab:IPFP-4}, $r_{n,o}$ represents the value in nth row and oth column of Table~\ref{tab:IPFP-3}, $s_c(o)$ represents the oth column sum of Table~\ref{tab:IPFP-4} and $s_r(o)$ represents the oth row sum of Table~\ref{tab:IPFP-3}.
\end{enumerate}

The iteration stops after Table~\ref{tab:IPFP-4}, as the column and row sums of this table are equal to that of Table~\ref{tab:IPFP-1}, up to 3 decimal places. 

\begin{table}
\begin{tabular}{>{$}l<{$}  l  l >{\bfseries}l }
\toprule
h & $\bm{l_7(h)}$ & $\bm{l_8(h)}$ & TOTAL\\ 
\midrule
1& 10 & 10 & 20\\  
\midrule
2 & 5 & 10 & 15\\
\midrule
3 & 17.3205 & 14.1421 & 	31.4626\\
\midrule
4 & 15 & 20 & 35\\
\midrule
5 & 4 & 10 & 14\\
\midrule
6 & 8.0718 & 10 & 18.0718\\
\midrule
7 & 8 & 18.478 & 26.478\\
\midrule
8 & 8 & 10 & 18 \\
\midrule
9 & 4 & 10 & 14 \\
\midrule
10 & 13.4944 & 14.1421 & 27.6365 \\
\midrule
11 & 12 & 20 & 32 \\
\midrule
12 & 4.899 & 10 & 14.899 \\
\midrule
13 & 7.101 & 10 & 17.101 \\
\midrule
14 & 8.026 & 18.478 & 26.504 \\
\midrule
15 & 10 & 10 & 20 \\
\midrule
16 & 7.0711 & 10 & 17.0711 \\
\midrule
17 & 7.0711 & 14.1421 & 21.2132 \\
\midrule
18 & 10 & 20 & 30 \\
\midrule
19 & 4 & 10 & 14 \\
\midrule
20 & 6 & 10 & 16 \\
\midrule
21 & 7.7274 & 18.478 & 26.2054 \\
\midrule
22 & 8 & 10 & 18 \\
\midrule
23 & 12 & 10 & 22 \\
\midrule
24 & 8 & 14.1421 & 22.1421 \\
\midrule
25 & 12 & 20 & 32 \\
\midrule
26 & 2.8284 & 10 & 12.8284 \\
\midrule
27 & 9.1716 & 10 & 19.1716 \\
\midrule
28 & 5.2264 & 18.478 & 23.7044 \\
\midrule
\bf{TOTAL} & \bf{234.0087} & \bf{370.4805} & 604.4892 \\ 	
\bottomrule
\end{tabular}
\caption{Table whose rows are populated with corresponding elements of sets $L_7$ and $L_8$. This table is constructed to compute the Row and Column totals that needs to be maintained in Row Fitting and Column Fitting respectively}
\label{tab:IPFP-1}
\end{table}

\begin{table}
\begin{tabular}{>{$}l<{$}  l  l >{\bfseries}l  }
\toprule
h & $\bm{l_7'(h)}$ & $\bm{l_8'(h)}$ & TOTAL\\ 
\midrule
1 & 1 & 1 & 2\\  
\midrule
2 & 1 & 1 & 2\\
\midrule
3 & 1 & 1 & 	2\\
\midrule
4 & 1 & 1 & 2\\
\midrule
5 & 1 & 1 & 2\\
\midrule
6 & 1 & 1 & 2\\
\midrule
7 & 1 & 1 & 2\\
\midrule
8 & 1 & 1 & 2 \\
\midrule
9 & 1 & 1 & 2 \\
\midrule
10 & 1 & 1 & 2 \\
\midrule
11 & 1 & 1 & 2 \\
\midrule
12 & 1 & 1 & 2 \\
\midrule
13 & 1 & 1 & 2 \\
\midrule
14 & 1 & 1 & 2 \\
\midrule
15 & 1 & 1 & 2 \\
\midrule
16 & 1 & 1 & 2 \\
\midrule
17 & 1 & 1 & 2 \\
\midrule
18 & 1 & 1 & 2 \\
\midrule
19 & 1 & 1 & 2 \\
\midrule
20 & 1 & 1 & 2 \\
\midrule
21 & 1 & 1 & 2 \\
\midrule
22 & 1 & 1 & 2 \\
\midrule
23 & 1 & 1 & 2 \\
\midrule
24 & 1 & 1 & 2 \\
\midrule
25 & 1 & 1 & 2 \\
\midrule
26 & 1 & 1 & 2 \\
\midrule
27 & 1 & 1 & 2 \\
\midrule
28 & 1 & 1 & 2 \\
\midrule
\bf{TOTAL} & \bf{28} & \bf{28} & 56 \\ 
\bottomrule
\end{tabular}
\caption{Initial table. The table on which the Row fitting of the first iteration.}
\label{tab:IPFP-2}
\end{table}

\begin{table}[h]
\resizebox{\textwidth}{!}{%
\begin{subtable}[h]{0.6\linewidth}
\centering
\begin{tabular}{>{$}l<{$}  l  l >{\bfseries}l  }
\toprule
$h$ & $\bm{l_7'(h)}$ & $\bm{l_8'(h)}$ & {\sc total}\\ 
\midrule
1 & 10 & 10 & 20\\  
\midrule
2 & 7.5 & 7.5 & 15\\
\midrule
3 & 15.7313 & 15.7313 & 	31.4626\\
\midrule
4 & 17.5 & 17.5 & 35\\
\midrule
5 & 7 & 7 & 14\\
\midrule
6 & 9.0359 & 9.0359 & 18.0718\\
\midrule
7 & 13.239 & 13.239 & 26.478\\
\midrule
8 & 9 & 9 & 18 \\
\midrule
9 & 7 & 7 & 14 \\
\midrule
10 & 13.81825 & 3.818251 & 27.6365 \\
\midrule
11 & 16 & 16 & 32 \\
\midrule
12 & 7.4495 & 7.4995 & 14.899 \\
\midrule
13 & 8.5505 &  8.5505 & 17.101 \\
\midrule
14 & 13.2520 & 13.2520 & 26.5040 \\
\midrule
15 & 10 & 10 & 20 \\
\midrule
16 & 8.53555 & 8.53555 & 17.0711 \\
\midrule
17 & 10.6066 & 10.6066 & 21.2132 \\
\midrule
18 & 15 & 15 & 30 \\
\midrule
19 & 7 & 7 & 14 \\
\midrule
20 & 8 & 8 & 16 \\
\midrule
21 & 13.1027 & 13.1027 & 26.2054 \\
\midrule
22 & 9 & 9 & 18 \\
\midrule
23 & 11 & 11 & 22 \\
\midrule
24 & 11.07105 & 11.07105 & 22.1421 \\
\midrule
25 & 16 & 16 & 32 \\
\midrule
26 & 6.4142 & 6.142 & 12.8284 \\
\midrule
27 & 9.5858 & 9.5858 & 19.1716 \\
\midrule
28 & 11.8522  & 11.8522 & 23.7044 \\
\midrule
{\sc total} & \bf{302.2446} & \bf{302.2446} & 604.4892 \\ 
\bottomrule
\end{tabular}
\caption{Table obtained on performing row fitting.}
\label{tab:IPFP-3}
\end{subtable}%
\hfill
\begin{subtable}[h]{0.6\linewidth}
\centering
\begin{tabular}{llll}
\toprule
$h$ & $\bm{l_7'(h)}$ & $\bm{l_8'(h)}$ & TOTAL\\ 
\midrule
1 & 7.7424 & 12.2576 & 20\\  
\midrule
2 & 5.8068 & 9.1932 & 15\\
\midrule
3 & 12.1798 & 19.2829 & 	31.4627\\
\midrule
4 & 13.5491 & 21.4509 & 35\\
\midrule
5 & 5.4197 & 8.5803 & 14\\
\midrule
6 & 6.9959 & 11.0759 & 18.0718\\
\midrule
7 & 10.2501 & 16.2279 & 26.478\\
\midrule
8 & 6.9681 & 11.0319 & 18 \\
\midrule
9 & 5.4197 & 8.5803 & 14 \\
\midrule
10 & 10.6986 & 16.9379 & 27.6365 \\
\midrule
11 & 12.3878 & 19.6122 & 32 \\
\midrule
12 & 5.7677 & 9.1313 & 14.899 \\
\midrule
13 & 6.6201 & 10.4809 & 17.101 \\
\midrule
14 & 10.2602 & 16.2438 & 26.504 \\
\midrule
15 & 7.7424 & 12.2576 & 20 \\
\midrule
16 & 6.6085 & 10.4625 & 17.071 \\
\midrule
17 & 8.212 & 13.0012 & 21.2132 \\
\midrule
18 & 11.6135 & 18.3865 & 30 \\
\midrule
19 & 5.4197 & 8.5803 & 14 \\
\midrule
20 & 6.1939 & 9.8061 & 16 \\
\midrule
21 & 10.1446 & 16.0608 & 26.2054 \\
\midrule
22 & 6.9681 & 11.0319 & 18 \\
\midrule
23 & 8.5166 & 13.4834 & 22 \\
\midrule
24 & 8.5716 & 13.5705 & 22.1421 \\
\midrule
25 & 12.3878 & 19.6122 & 32 \\
\midrule
26 & 4.9661 & 7.8623 & 12.8284 \\
\midrule
27 & 7.4217 & 11.7499 & 19.1716 \\
\midrule
28 & 9.1764 & 14.5280 & 23.7044 \\
\midrule
{\sc total} & \bf{234.0087} & \bf{370.4805} & 604.4892 \\ 
\bottomrule
\end{tabular}
\caption{Table obtained on performing column fitting.}
\label{tab:IPFP-4}
\end{subtable}
}
\caption{IPFP First Iteration}
\label{tab:IPFPIteration1}
\end{table}
\FloatBarrier
\begin{flalign*}
\qquad & \text{Considering any row from the Table~\ref{tab:IPFP-4}, we compute $m$} & \\
& m = \frac{12.2576}{7.7424} = 1.5832 &\\
& y =  1.5832 x \text{, equation of the expected line}&
\end{flalign*}

Figure~\ref{fig:Comp3-edge} shows the relation between the
corresponding elements of $L_7$ and $L_8$. It also indicates, for each
pair $\langle l_7(h), l_8(h) \rangle$ the corresponding point,
$\langle l_7'(h), l_8'(h) \rangle$, on the $y=1.5985 x$ line. Further,
Table~\ref{tab:Comp3-edge} provides the legend for this figure.

\begin{figure}[h]
\centering
\includegraphics[width=0.6\textwidth]{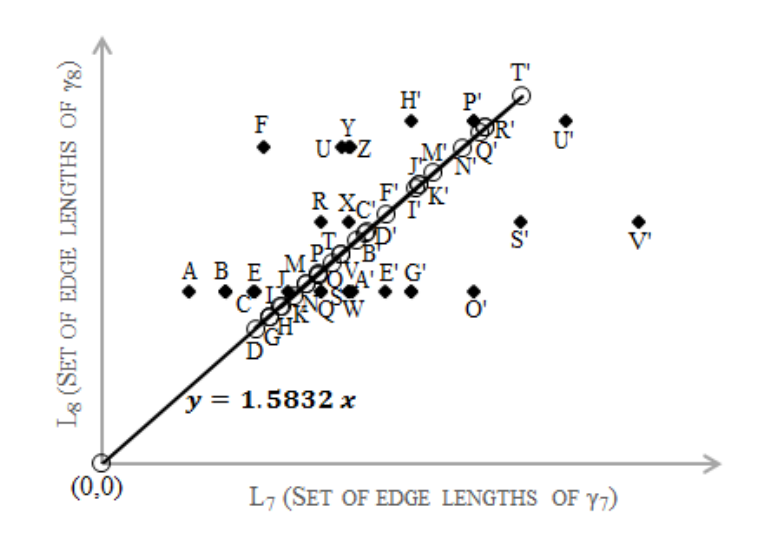}
\caption{\scriptsize Solid dots show the existing relation between $L_7$ and $L_8$, whereas the line passing through the hollow dots shows the expected relation between $L_7$ and $L_8$}
\label{fig:Comp3-edge}
\end{figure}
\FloatBarrier
\begin{longtable}{ >{$}l<{$} >{$}l<{$} }
{\!\begin{aligned} 
               A &= \langle 2.8284, 10 \rangle  \\   
                  &\in \{ \langle l_7(26), l_8(26) \rangle \} \end{aligned}} & {\!\begin{aligned} 
               B &= \langle 4, 10 \rangle  \\   
                  &\in \{ \langle l_7(5), l_8(5) \rangle, \langle l_7(9), l_8(9) \rangle, \langle l_7(19), l_8(19) \rangle \} \end{aligned}} \\
{\!\begin{aligned} 
               C &= \langle 4.899, 10 \rangle  \\   
                  &\in \{ \langle l_7(12), l_8(12) \rangle \} \end{aligned}} & {\!\begin{aligned} 
               D &= \langle 4.9661	, 7.8623 \rangle  \\   
                  &\in \{ \langle l_7'(26), l_8'(26) \rangle \} \end{aligned}} \\    
{\!\begin{aligned} 
               E &= \langle 5, 10 \rangle  \\   
                  &\in \{ \langle l_7(2), l_8(2) \rangle \} \end{aligned}} & {\!\begin{aligned} 
               F &= \langle 5.2264, 18.478 \rangle  \\   
                  &\in \{ \langle l_7(28), l_8(28) \rangle \} \end{aligned}}\\      
{\!\begin{aligned} 
              G &= \langle 5.4197, 8.5803 \rangle  \\   
                  &\in \{ \langle l_7'(5), l_8'(5) \rangle, \langle l_7'(9), l_8'(9), \langle l_7'(19), l_8'(19) \rangle \} \end{aligned}} & {\!\begin{aligned} 
              H &= \langle 5.7677, 9.1313 \rangle  \\   
                  &\in \{ \langle l_7'(12), l_8'(12) \rangle \} \end{aligned}} \\     
{\!\begin{aligned} 
               I &= \langle 5.8068, 9.1932 \rangle  \\   
                  &\in \{ \langle l_7'(2), l_8'(2) \rangle \} \end{aligned}} & {\!\begin{aligned} 
               J &= \langle 6, 10 \rangle  \\   
                  &\in \{ \langle l_7(20), l_8(20) \rangle \} \end{aligned}} \\   
{\!\begin{aligned} 
              K &= \langle 6.1939,	9.8061 \rangle  \\   
                  &\in \{ \langle l_7'(20), l_8'(20) \rangle \} \end{aligned}} & {\!\begin{aligned} 
               M &= \langle 6.6085	, 10.4625 \rangle  \\   
                  &\in \{ \langle l_7'(16), l_8'(16) \rangle \} \end{aligned}} \\                             
{\!\begin{aligned} 
               N &= \langle 6.6201	, 10.4809 \rangle  \\   
                  &\in \{ \langle l_7'(13), l_8'(13) \rangle \} \end{aligned}} & {\!\begin{aligned} 
               O &= \langle 6.9681	, 11.0319 \rangle  \\   
                  &\in \{ \langle l_7'(8), l_8'(8) \rangle, \langle l_7'(22), l_8'(22) \rangle \} \end{aligned}} \\
{\!\begin{aligned} 
               P &= \langle 6.9959, 11.0759 \rangle  \\   
                  &\in \{ \langle l_7'(6), l_8'(6) \rangle \} \end{aligned}} & {\!\begin{aligned} 
               Q &= \langle 7.0711, 10 \rangle  \\   
                  &\in \{ \langle l_7(16), l_8(16) \rangle \} \end{aligned}} \\ 
{\!\begin{aligned} 
               R &= \langle 7.0711, 14.1421 \rangle  \\   
                  &\in \{ \langle l_7(17), l_8(17) \rangle \} \end{aligned}} & {\!\begin{aligned} 
               S &= \langle 7.101, 10 \rangle  \\   
                  &\in \{ \langle l_7(13), l_8(13) \rangle \} \end{aligned}} \\                     
{\!\begin{aligned} 
               T &= \langle 7.4217, 11.7499 \rangle  \\   
                  &\in \{ \langle l_7'(27), l_8'(27) \rangle \} \end{aligned}} & {\!\begin{aligned} 
               U &= \langle 7.7274	, 18.4780 \rangle  \\   
                  &\in \{ \langle l_7(21), l_8(21) \rangle \} \end{aligned}} \\             
{\!\begin{aligned} 
               V &= \langle 7.7424, 12.2576 \rangle  \\   
                  &\in \{ \langle l_7'(1), l_8'(1) \rangle, \langle l_7'(15), l_8'(15) \rangle \} \end{aligned}} & {\!\begin{aligned} 
               W &= \langle 8, 10 \rangle  \\   
                  &\in \{ \langle l_7(8), l_8(8) \rangle, \langle l_7(22), l_8(22) \rangle \} \end{aligned}} \\  
{\!\begin{aligned} 
               X &= \langle 8, 14.1421 \rangle  \\   
                  &\in \{ \langle l_7(24), l_8(24) \rangle \} \end{aligned}} & {\!\begin{aligned} 
               Y &= \langle 8, 18.478 \rangle  \\   
                  &\in \{ \langle l_7(7), l_8(7) \rangle \} \end{aligned}}  \\  
{\!\begin{aligned} 
               Z &= \langle 8.0260, 18.4780 \rangle  \\   
                  &\in \{ \langle l_7(14), l_8(14) \rangle \} \end{aligned}} & {\!\begin{aligned} 
               A' &= \langle 8.0718, 10 \rangle  \\   
                  &\in \{ \langle l_7(6), l_8(6) \rangle \} \end{aligned}} \\  
{\!\begin{aligned} 
               B' &= \langle 8.2120	, 13.0012 \rangle  \\   
                  &\in \{ \langle l_7'(17), l_8'(17) \rangle \} \end{aligned}} & {\!\begin{aligned} 
               C' &= \langle 8.5166, 13.4834 \rangle  \\   
                  &\in \{ \langle l_7'(23), l_8'(23) \rangle \} \end{aligned}} \\          
{\!\begin{aligned} 
               D' &= \langle 8.5716, 13.5705 \rangle  \\   
                  &\in \{ \langle l_7'(24), l_8'(24) \rangle \} \end{aligned}} & {\!\begin{aligned} 
               E' &= \langle 8.521, 13.6211 \rangle  \\   
                  &\in \{ \langle l_7'(24), l_8'(24) \rangle \} \end{aligned}} \\       
{\!\begin{aligned} 
               F' &= \langle 9.1764	, 14.5280 \rangle  \\   
                  &\in \{ \langle l_7'(28), l_8'(28) \rangle \} \end{aligned}} & {\!\begin{aligned} 
               G' &= \langle 10, 10 \rangle  \\   
                  &\in \{ \langle l_7(1), l_8(1) \rangle, \langle l_7(15), l_8(15) \rangle \} \end{aligned}} \\       
{\!\begin{aligned} 
               H' &= \langle 10, 20 \rangle  \\   
                  &\in \{ \langle l_7(18), l_8(18) \rangle \} \end{aligned}} & {\!\begin{aligned} 
               I' &= \langle 10.1446, 16.0608 \rangle  \\   
                  &\in \{ \langle l_7'(21), l_8'(21) \rangle \} \end{aligned}} \\                          
{\!\begin{aligned} 
               J' &= \langle 10.2501, 16.2279 \rangle  \\   
                  &\in \{ \langle l_7'(7), l_8'(7) \rangle \} \end{aligned}} & {\!\begin{aligned} 
               K' &= \langle 10.2602, 16.2438 \rangle  \\   
                  &\in \{ \langle l_7'(14), l_8'(14) \rangle \} \end{aligned}} \\                                                                                                                                                                                                                                                                                       
{\!\begin{aligned} 
               M' &= \langle 10.6986, 16.9379 \rangle  \\   
                  &\in \{ \langle l_7'(10), l_8'(10) \rangle \} \end{aligned}} & {\!\begin{aligned} 
               N' &= \langle 11.6135, 18.3865 \rangle  \\   
                  &\in \{ \langle l_7'(18), l_8'(18) \rangle \} \end{aligned}} \\                             
{\!\begin{aligned} 
               O' &= \langle 12, 10 \rangle  \\   
                  &\in \{ \langle l_7(23), l_8(23) \rangle \} \end{aligned}} & {\!\begin{aligned} 
               P' &=  \langle 12, 20 \rangle  \\   
                  &\in \{ \langle l_7(11), l_8(11) \rangle, \langle l_7(25), l_8(25) \rangle \} \end{aligned}} \\           
{\!\begin{aligned} 
               Q' &= \langle 12.1798, 19.2829 \rangle  \\   
                  &\in \{ \langle l_7'(3), l_8'(3) \rangle \} \end{aligned}} & {\!\begin{aligned} 
               R' &= \langle 12.3878	19.6122 \rangle  \\   
                  &\in \{ \langle l_7'(11), l_8'(11) \rangle, \langle l_7'(25), l_8'(25) \rangle \} \end{aligned}} \\              
{\!\begin{aligned} 
               S' &=  \langle 13.4944, 14.1421 \rangle  \\   
                  &\in \{ \langle l_7(10), l_8(10) \rangle \} \end{aligned}} & {\!\begin{aligned} 
               T' &= \langle 13.5491	21.4509 \rangle  \\   
                  &\in \{ \langle l_7'(4), l_8'(4) \rangle \} \end{aligned}} \\     
{\!\begin{aligned} 
               U' &= \langle 15, 20 \rangle  \\   
                  &\in \{ \langle l_7(4), l_8(4) \rangle \} \end{aligned}} & {\!\begin{aligned} 
               V' &= \langle 17.3205, 14.1421 \rangle  \\   
                  &\in \{ \langle l_7(3), l_8(3) \rangle \} \end{aligned}} \\        
\caption{Legend of Figure~\ref{fig:Comp3-edge}}  
\label{tab:Comp3-edge}                                                                                                          
\end{longtable}
\FloatBarrier

Now, the Euclidean Distance is computed using \eqref{euclideanEqn} 

\begin{tabular}{ >{$}l<{$} @{\;=\;} >{$}l<{$} >{$}l<{$} @{\;=\;} >{$}l<{$} >{$}l<{$} @{\;=\;} >{$}l<{$} >{$}l<{$} @{\;=\;} >{$}l<{$}}
\Delta_{7,8}(1) & 3.1928 & \Delta_{7,8}(2) & 1.1409 & \Delta_{7,8}(3) & 7.2701 & \Delta_{7,8}(4) & 2.0518 \\ 
\Delta_{7,8}(5) & 2.0077 & \Delta_{7,8}(6) & 1.5215 & \Delta_{7,8}(7) & 3.1821 & \Delta_{7,8}(8) & 1.2285 \\  
\Delta_{7,8}(9) & 2.0077 & \Delta_{7,8}(10) & 3.9539 & \Delta_{7,8}(11) & 0.5484 & \Delta_{7,8}(12) & 1.1803 \\
\Delta_{7,8}(13) & 0.6801 & \Delta_{7,8}(14) & 3.1596 & \Delta_{7,8}(15) & 3.1928 & \Delta_{7,8}(16) & 0.6541\\
\Delta_{7,8}(17) & 1.6135 & \Delta_{7,8}(18) & 2.2819 & \Delta_{7,8}(19) & 2.0077 & \Delta_{7,8}(20) & 0.2742  \\
\Delta_{7,8}(21) & 3.4184 & \Delta_{7,8}(22) & 1.4593 & \Delta_{7,8}(23) & 4.9263 & \Delta_{7,8}(24) & 0.8084 \\
\Delta_{7,8}(25) & 0.5484 & \Delta_{7,8}(26) & 3.0231 & \Delta_{7,8}(27) & 2.4748 & \Delta_{7,8}(28) & 5.5861 \\
\end{tabular}
\begin{flalign*}
\qquad & \sum_{h=1}^{28} \Delta_{7,8}(h) =65.6736 & \\
& \bm{\rho(\gamma_7,\gamma_8)} = 0.7011 \text{, using \eqref{rhoEqn}}&
\end{flalign*}

Computing $d(\gamma_7,\gamma_8)$
\begin{flalign*}
\qquad & \bm{d(\gamma_7,\gamma_8)} = 0.8106 \text{, with } \beta = 0.5 \text{, using \eqref{dDefinitionEqn}}&
\end{flalign*}
\FloatBarrier


\end{document}